\pdfoutput=1

\documentclass[
 twocolumn,
 5p,
 12pt,%
]{elsarticle}

\usepackage{amssymb}
\usepackage{gensymb}

\usepackage{amsmath}

\usepackage{amsthm}

\usepackage{xspace}

\usepackage{listings}
\lstdefinestyle{block}{language=C++,%
  basicstyle=\ttfamily\scriptsize,%
  tabsize=2,%
  breaklines,%
  }

\lstnewenvironment{cpp}[1][]{%
\lstset{style=block,#1}}{}

\allowdisplaybreaks

\usepackage{tikz}
\usetikzlibrary{arrows}

\iffalse
\usetikzlibrary{calc}
\usetikzlibrary{chains}
\usetikzlibrary{decorations}
\usetikzlibrary{matrix}
\usetikzlibrary{positioning}
\usetikzlibrary{scopes}

\usetikzlibrary{calc,fadings,decorations.pathreplacing}
\usepackage{tikz-3dplot}

\pgfrealjobname{paper}

\newcommand\inputTikZ[1]{
\beginpgfgraphicnamed{#1}
\input{#1.tikz}
\endpgfgraphicnamed
}

\else

  \newcommand\inputTikZ[1]{
    \includegraphics{#1}
  }

\fi

\usepackage{ifpdf}
\ifpdf
  \usepackage{graphicx}

  \usepackage[pdftex]{hyperref}

\else
  \usepackage{graphicx}
  \usepackage{bmpsize}
  \DeclareGraphicsExtensions{.jpg,.jpeg,.png,.gif,.eps,.ps,.eps.gz,.ps.gz,.eps.Z,.pdf}
  
  \usepackage[hypertex]{hyperref}

\fi

\graphicspath{{rss_experiments/}{ukf-plots/}{tikz-figures/}}

\iffalse
\usepackage{pstool}
\else
\newcommand\psfragfig[2][]{
  \includegraphics[#1]{#2}
}
\fi

\theoremstyle{plain}

\newtheorem{lem}{Lemma} 
\newtheorem{corl}{Corollary}

\theoremstyle{definition}

\newtheorem{dfn}{Definition}

\newcommand*\R{\ensuremath{\mathbb{R}}} 
\newcommand*\Z{\ensuremath{\mathbb{Z}}} 
\newcommand*\PP{\ensuremath{\mathbb{P}}} 

\renewcommand*\d{\mathrm{d}}
\newcommand*\dx{\mathrm{d}x}

\DeclareMathOperator\id{id} 
\DeclareMathOperator\I{I}   
\newcommand*\trans[1][]{^{#1\top}} 
\newcommand*\inv{^{-1}}

\newcommand*\norm[1]{\left\| #1 \right\|} 
\newcommand*\abs[1]{\left| #1 \right|} 
\newcommand*\floor[1]{\left\lfloor #1 \right\rfloor}

\newcommand*\outerProd[1]{#1#1\trans}

\DeclareMathOperator\E{\operatorname{E}} 
\DeclareMathOperator\Cov{\operatorname{Cov}} 
\DeclareMathOperator\cor{\operatorname{cor}} 
\let\P\relax
\DeclareMathOperator\P{\operatorname{P}} 
\DeclareMathOperator\p{\operatorname{p}} 
\DeclareMathOperator\Nd{\mathcal{N}} 
\newcommand*\Ndp[2]{
\Nd\!\left(#1,\,#2\right)
}

\newcommand\Nde[2]{\expp{-\tfrac12\norm{#1}^2_{#2}}}

\newcommand\F{\mathcal{F}}

\newcommand*\expp[1]{\exp\!\left(#1\right)}
\newcommand*\imp{\Rightarrow} 
\newcommand*\eqv{\Leftrightarrow} 
\newcommand*\eps{\varepsilon} 
\let\phi\varphi 
\newcommand*\tX{\tilde X}
\newcommand*\teps{\tilde\eps}

\DeclareMathOperator\atan{atan}
\DeclareMathOperator\atantwo{atan2}
\DeclareMathOperator\sinc{sinc}

\newcommand\olog{\operatorname{\overline{\log}}}

\let\arccos\acos

\newcommand{\argmin}{\operatornamewithlimits{argmin}}
 
\newcommand*\setwith[2]{\left\{#1\mid #2\right\}}
\newcommand*\setminusset[1]{\setminus\left\{#1\right\}}

\newcommand*\matrxs[1]{\left[\begin{smallmatrix} #1 \end{smallmatrix}\right]}

\DeclareRobustCommand*\mplus{\texorpdfstring{\ensuremath{\mathbin{\text{\raisebox{-.2ex}{$\boxplus$}}}}}{[+]}\xspace}
\DeclareRobustCommand*\mmnus{\ensuremath{\mathbin{\text{\raisebox{-.2ex}{$\boxminus$}}}}\xspace}

\newcommand*\SO[1]{\texorpdfstring{\ensuremath{SO(#1)}}{SO(#1)}}
\newcommand*\SE[1]{\texorpdfstring{\ensuremath{SE(#1)}}{SE(#1)}}
\let\S\relax 
\newcommand*\S{\mathcal{S}}
\newcommand*\M{\mathcal{M}}

\newcommand*\Rn{\texorpdfstring{$\R^n$}{R\textasciicircum n}\xspace}

\newcommand\tquat[2]{\bigl[\!
\begin{smallmatrix} #1 \\ #2 
\end{smallmatrix}\!\bigr]}

\newcommand\dquat[2]{\left[
\begin{matrix} #1 \\ #2 
\end{matrix}\right]}

\newcommand\quat[2]{
\mathchoice{\dquat{#1}{#2}}
           {\tquat{#1}{#2}}
           {\tquat{#1}{#2}}
           {\tquat{#1}{#2}}
}

\newcommand\xfrac[2]{%
\hbox{\raisebox{.9ex}{\ensuremath{#1}}%
{\raisebox{-.1ex}{\Large /}}%
\raisebox{-.5ex}{\ensuremath{#2}}%
}%
}

\newcommand*\cf{cf.\@\xspace}
\newcommand*\ie{i.e.\@\xspace}
\newcommand*\eg{e.g.\@\xspace}

\newcommand*\bpm{\mplus-method\xspace}
\newcommand*\mpm{\mplus/\mmnus}

\newcommand{\meanofsigmap}{\ensuremath{\mbox{\sc MeanOfSigmaPoints}}\xspace}

\newcommand{\sigp}[2]{\ensuremath{\mathcal{#1}^{[#2]}}}

\newcommand*{\sip}{\emph{sigma point}\xspace}
\newcommand*{\sips}{\emph{sigma points}\xspace}
\newcommand*{\ut}{\emph{unscented transform}\xspace}

\newcommand{\into}{\hookrightarrow}

\newcommand\normF[1]{\norm{#1}_{\mathrm{F}}}

\newcommand\defi[1]{\emph{#1}}

\newcommand\trace{\operatorname{tr}}

\makeatletter
\newcommand*\stckrel[3]{%
  \setboxz@h{$\displaystyle\mathop{#3}$}
  \dimen@-\wd\z@
  \global\setboxz@h{%
  $\displaystyle%
  \mathop{#3}\limits^{#1}_{#2}$}
  \advance\dimen@\wd\z@
  \global\divide\dimen@ 2%
  \hskip\dimen@&\hskip-\dimen@%
  \mathrel{\box\z@}%
}%
\makeatother

\usepackage[outerbars, color]{changebar}
\cbcolor{red}
\iffalse
  \newenvironment{modified}[1][]{\cbstart}{\cbend}
\else
  \newenvironment{modified}[1][]{}{}
\fi

\makeatletter
\def\ps@pprintTitle{%
     \let\@oddhead\@empty
     \let\@evenhead\@empty
     \let\@oddfoot\@empty
     \let\@evenfoot\@empty}
\makeatother

\begin{document}

\begin{frontmatter}

\title{Integrating Generic Sensor Fusion Algorithms with Sound~State~Representations through Encapsulation of Manifolds
}


\author[SFB,UNI]{Christoph Hertzberg%
}
\ead{chtz@informatik.uni-bremen.de 
}
\author[DFKI,UNI]{Ren\'e Wagner}
\author[SFB,UNI,DFKI]{Udo Frese}
\author[DFKI,UNI]{Lutz Schr\"oder}

\address[SFB]{%
SFB/TR~8 -- Spatial Cognition. Reasoning, Action, Interaction 
}

\address[UNI]{%
Fachbereich~3 -- Mathematik und Informatik, 
Universit\"at  Bremen, Postfach 330~440, 28334~Bremen, Germany}

\address[DFKI]{%
Deutsches Forschungszentrum f\"ur K\"unstliche Intelligenz (DFKI), 
Sichere Kognitive Systeme,\break 
Enrique-Schmidt-Str.~5, 28359~Bremen, Germany}

\begin{abstract}
  Common estimation algorithms, such as least squares estimation or
  the Kalman filter, operate on a state in a state space $\S$ that is
  represented as a real-valued vector. However, for many quantities,
  most notably orientations in 3D, $\S$ is not a vector space, but a
  so-called manifold, \ie it behaves like a vector space locally but
  has a more complex global topological structure. For integrating
  these quantities, several ad-hoc approaches have been proposed.

  Here, we present a principled solution to this problem where the
  structure of the manifold $\S$ is encapsulated by two operators,
  state displacement $\mplus:\S\times\R^n\to\S$ and its inverse
  $\mmnus:\S\times{}\S\to\R^n$. These operators provide a local
  vector-space view $\delta\mapsto{}x\mplus\delta$ around a given
  state $x$. Generic estimation algorithms can then work on the
  manifold $\S$ mainly by replacing $+/-$ with \mpm where appropriate.
  We analyze these operators axiomatically, and demonstrate their use
  in least-squares estimation and the Unscented Kalman
  Filter. Moreover, we exploit the idea of encapsulation from a
  software engineering perspective in the \emph{Manifold Toolkit},
  where the \mpm operators mediate between a ``flat-vector'' view for
  the generic algorithm and a ``named-members'' view for the problem
  specific functions.
\end{abstract}


\end{frontmatter}


\section{Introduction}

Sensor fusion is the process of combining information
obtained from a variety of different sensors into a joint belief over
the system state.  In the design of a sensor fusion system, a key
engineering task lies in finding a state representation that (a)
adequately describes the relevant aspects of reality and is (b)
compatible with the sensor fusion algorithm in the sense that the
latter yields meaningful or even optimal results when operating on the
state representation.
\begin{figure}[tb!]
\centering
\begin{modified} 
  \inputTikZ{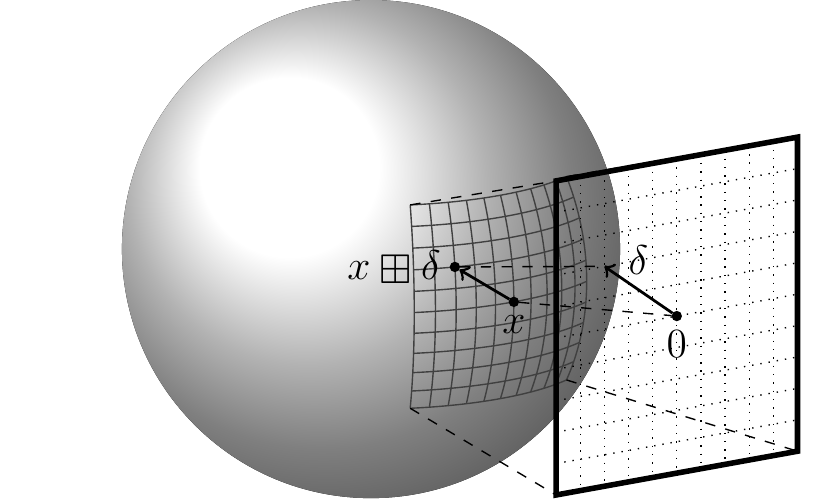}
\end{modified}
  \caption{Mapping a local neighborhood in the state space (here: on
    the unit sphere $S^2$) into $\R^n$ (here: the plane) allows for
    the use of standard sensor fusion algorithms without explicitly
    encoding the global topological structure.}
  \label{fig:sphere}
\end{figure}

Satisfying both these goals at the same time has been a long-standing
challenge.  Standard sensor fusion algorithms typically operate on
real valued vector state representations ($\R^n$) while mathematically
sound representations often form more complex, non-Euclidean
topological spaces.  A very common example of this comes up, \eg
within the context of inertial navigation systems (INS) where a part
of the state space is \SO3, the group of orientations in $\R^3$.  To
estimate variables in \SO3, there are generally two different
approaches.  The first uses a parameterization of minimal dimension,
\ie with three parameters (\eg Euler angles), and operates on the
parameters like on $\R^3$.  This parameterization has singularities,
\ie a situation analogous to the well-known \emph{gimbal lock} problem
in gimbaled INS \cite{Grewal01} can occur where very large changes in
the parameterization are required to represent small changes in the
state space.  Workarounds for this exist that try to avoid these parts
of the state space, as was most prominently done in the guidance system
of the Apollo Lunar Module \cite{Klumpp1974}, or switch between
alternative orderings of the parameterization each of which exhibit
singularities in different areas of the state space.  The second
alternative is to overparameterize states with a non-minimal
representation such as unit quaternions or rotation matrices which are
treated as $\R^4$ or $\R^{3\times 3}$ respectively and re-normalized
as needed \cite{Wheeler95iterativeestimation, EKFQuat}.  This has
other disadvantages such as redundant parameters or degenerated, non-normalized
variables.

Both approaches require representation-specific modifications of the
sensor fusion algorithm and tightly couple the state representation
and the sensor fusion algorithm, which is then no longer a generic
black box but needs to be adjusted for every new state representation.

Our approach is based on the observation that sensor fusion algorithms
employ operations which are inherently local, \ie they compare and
modify state variables in a local neighborhood around some reference.
We thus arrive at a generic solution that bridges the gap between the
two goals above by viewing the state space as a manifold.  Informally
speaking, every point of a manifold has a neighborhood that can be
mapped bi-directionally to $\R^n$.  This enables us to use an
arbitrary manifold $\S$ as the state representation while the sensor
fusion algorithm only sees a locally mapped part of $\S$ in $\R^n$ at
any point in time.  For the unit sphere $S^2$ this is illustrated in
figure~\ref{fig:sphere}.

We propose to implement the mapping by means of two encapsulation
operators $\mplus$ (``boxplus'') and $\mmnus$ (``boxminus'') where
\begin{alignat}2
\mplus &: \S\times\R^n &&\to \S, \\
\mmnus &: \S\times \S  &&\to\R^n. 
\end{alignat}
Here, $\mplus$ takes a manifold state and a small change expressed in
the mapped local neighborhood in $\R^n$ and applies this change to
the state to yield a new, modified state. Conversely, $\mmnus$
determines the mapped difference between two states.

The encapsulation operators capture an important duality: The generic
sensor fusion algorithm uses $\mmnus$ and $\mplus$ in place of the
corresponding vector operations $-$ and $+$ to compare and to modify
states, respectively -- based on flattened perturbation vectors -- and
otherwise treats the state space as a black box.
Problem-specific code such as measurement models, on the other
hand, can work inside the black box, and use the most natural
representation for the state space at hand. The operators $\mplus$ and
$\mmnus$ translate between these alternative views.  We will later
show how this can be modeled in an implementation framework such that
manifold representations (with matching operators) even of very
sophisticated compound states can be generated automatically from a
set of manifold primitives (in our C++ implementation currently
$\R^n$, \SO2, \SO3, and $S^2$).

This paper extends course material \cite{tdsf06} and several master's
theses~\cite{Kurlbaum2007, Birbach2008, Hertzberg2008, Wagner2010}.
\begin{modified}{}
It starts with a discussion of related work in
Section~\ref{sec:related}.  Section~\ref{sec:BoxplusMethod} then
introduces the \mplus-method and 3D orientations as the most important application, and Section~\ref{sec:application} lays
out how least squares optimization and Kalman filtering can be
modified to operate on these so-called \mplus-manifolds.  
Section~\ref{sec:software} introduces the
aforementioned software toolkit, and Section~\ref{sec:experiments}
shows practical experiments. Finally, the appendices prove the
properties of \mplus-manifolds claimed in Section~\ref{sec:BoxplusMethod}
and give \mplus-manifold representations of the most 
relevant manifolds $\R/2\pi\Z$, $SO(n)$, $S^n$ and $\PP^n$ along with proofs.
\end{modified}


\section{Related Work}\label{sec:related}

\subsection{Ad-hoc Solutions}

 Several ad-hoc methods are available to integrate manifolds
into estimation algorithms working on $\R^n$~\cite{Schmidt01,Triggs00}. All of them have some
drawbacks, as we now discuss using 3D orientations as a running
example.

The most common workaround is to use a minimal parameterization (\eg
Euler angles)~\cite{Schmidt01} \cite[p.\,6]{Triggs00} and place the singularities in some part of the
workspace that is not used (\eg facing $90^\circ$ upwards). This,
however, creates an unnecessary constraint between the application and
the representation leading to a failure mode that is easily forgotten.
If it is detected, it requires a recovery strategy, in the worst case
manual intervention as in the aforementioned Apollo mission
\cite{Klumpp1974}.

Alternatively, one can switch between several minimal parameterizations
with different singularities. This works but is more complicated.

For overparameterizations (\eg quaternions), a normalization
constraint must be maintained (unit length), \eg by normalizing after
each step~\cite{Wheeler95iterativeestimation, EKFQuat}. The step
itself does not know about the constraint, and tries to improve the
fit by violating the constraint, which is then undone by normalization.
This kind of counteracting updates is inelegant and at least slows down
convergence. Lagrange multipliers could be used to enforce the constraint
exactly~\cite[p.40]{Triggs00} but lead to a more difficult equation system.
Alternatively, one could add the normalization constraint as a
``measurement'' but this is only an approximation, since it would need
to have uncertainty zero~\cite[p.40]{Triggs00}.

Instead, one could allow non-normalized states and apply a
normalization function every time before using them
\cite{LaViola03acomparison,Triggs00}. Then the normalization degree of freedom is redundant having no effect
on the cost function and the algorithm does not try to change it.
Some
algorithms can handle this (\eg
Levenberg-Marquardt~\cite[Chap.\@ 15]{Press92}) but many (\eg
Gauss-Newton~\cite[Chap.\@ 15]{Press92}) fail due to singular equations. This problem
can be solved by adding
the normalization constraint as a ``measurement'' with some arbitrary uncertainty.~\cite[p.40]{Triggs00}. 
Still, the estimated state has more
dimensions than necessary and computation time is
increased.

\subsection{Reference State with Perturbations}

Yet another way is to set a reference state and work relative to it
with a minimal parameterization~\cite[p.6]{Schmidt01,Triggs00}.  Whenever the parameterization
becomes too large, the reference system is changed accordingly
\cite{AbbeelThesis,Castellanos99}. This is practically similar to the
\bpm, where in $s\mplus\delta$ the reference state is $s$ and $\delta$
is the minimal parameterization. 
\begin{modified}
In particular \cite{Castellanos99} has
triggered our investigation, so we compare to their SPmap in detail in App.\@~\ref{sec:SPmap}.
The vague idea of applying perturbations in a way more specific than by simply adding
has been around in the literature for some time:
\end{modified}
\begin{quote}
   \emph{``We write this operation [the state update] as $x\rightarrow{}x+\delta{}x$, even though
   it may involve considerable more than vector addition.''}
   \hspace{0.5cm} W. Triggs~\cite[p.\,7]{Triggs00}
\end{quote}
\begin{modified}
Very recently, Strasdat et al.\@~\cite[Sec.\@~II.C]{strasdatrss2010} have summarized
this technique for $\SO3$ under the label Lie group/algebra representation.
This means the state is viewed as an element of the Lie group $\SO3$, i.e.\@ an orthonormal matrix or quaternion
but every step of the estimation algorithm operates in the Lie group's
tangential space, i.e.\@ the Lie algebra. The Lie's group's exponential 
maps a perturbation $\delta$ from the Lie algebra to the Lie group.
For Lie groups this is equivalent to our approach, with $s\mplus\delta=s\cdot\exp\delta$  (Sec.\@~\ref{sec:liegroups}).

Our contribution here compared to \cite{Castellanos99,Triggs00,strasdatrss2010} is to embed
this idea into a mathematical and algorithmic framework by means of an explicit axiomatization,
where the operators $\mplus$ and $\mmnus$ make it easy to adapt
algorithms operating on \Rn. Moreover, our framework is 
more generic, being applicable also 
to manifolds which fail to be
Lie groups, such as $S^2$.
\end{modified}

In summary, the discussion shows the need for a principled solution
that avoids singularities and needs neither normalization constraints
nor redundant degrees of freedom in the estimation algorithm.

\subsection{Quaternion-Based Unscented Kalman Filter}
\label{sec:ukfrelated}

\begin{modified}
A reasonable ad-hoc approach for handling quaternions in an EKF/UKF is to
normalize, updating the covariance with the Jacobian of the normalization
as in the EKF's dynamic step~\cite{davisoniccv2003}. 
This is problematic, because it makes 
the covariance singular. Still the EKF operates as the innovation
covariance remains positive definite. Nevertheless, it is unknown if
this phenomen causes problems, e.g.\@ the zero-uncertainty degree of freedom
could create overconfidence in another degree of freedom after a non-linear update.

With a similar motivation, van der Merwe~\cite{Merwe2004SPKF}
included a dynamic, $\dot q =\eta (1-|q|^2)q$ that drives
the quaternion $q$ towards normalization. In this case the measurement
update can still violate normalization, but the violation decays
over time.
\end{modified}

Kraft~\cite{Kraft2003} and Sipos~\cite{Sipos2008} propose a
method of handling quaternions in the state space of an Unscented
Kalman Filter~(UKF). Basically, they modify the \ut to work with
their specific quaternion state representation using a special
operation for adding a perturbation to a quaternion and for taking the
difference of two quaternions.

Our approach is related to Kraft's work in so far as we perform the
same computations for states containing quaternions.  It is more
general in that we work in a mathematical framework where the special
operations are encapsulated in the \mpm-operators.  This allows us to
handle not just quaternions, but general manifolds as representations
of both states and measurements, and to routinely adapt estimation
algorithms to manifolds without sacrificing their genericity.

\subsection{Distributions on Unit Spheres}

 The above-mentioned methods as well as the present work
treat manifolds locally as $\R^n$ and take care to accumulate small
steps in a way that respects the global manifold structure. This is an
approximation as most distributions, \eg Gaussians, are strictly
speaking not local but extend to infinity. In view of this problem,
several generalizations of normal distributions have been proposed
that are genuinely defined on a unit sphere.

The von Mises-Fisher distribution on $S^{n-1}$ was proposed by Fisher
\cite{Fisher1953} and is given by
\begin{equation}
\p(x) = C_n(\kappa)\expp{\kappa\mu\trans x}, 
\end{equation}
with $\mu, x\in S^{n-1}\subset\R^n$, $\kappa>0$ and a normalization constant
$C_n(\kappa)$. $\mu$ is called the \emph{mean direction} and $\kappa$ the
\emph{concentration parameter}.

For the unit-circle $S^1$ this reduces to the von Mises distribution
\begin{equation}
\p(x) = C_2(\kappa)\expp{\kappa\cos(x-\mu)},
\end{equation}
with $x,\mu\in[-\pi,\pi)$.

The von Mises-Fisher distribution locally looks like a normal distribution
$\Ndp{0}{\frac1\kappa\I_{n-1}}$ (viewed in the tangential space), 
where $\I_{n-1}$ is the $(n-1)$-dimensional unit matrix.
As only a single parameter exists to model the covariance, only isotropic
distributions can be modeled, \ie contours of constant probability are always
circular. Especially for posterior distributions arising in sensor fusion
this is usually not the case.

To solve this, \ie to represent general multivariate normal
distributions, Kent proposed the Fisher-Bingham distribution on the
sphere \cite{kent1982}. It is defined as:
\begin{equation}
\nonumber \p(x)=\tfrac{1}{c(\kappa,\beta)}\expp{\kappa\gamma_1\trans x
+ \beta[(\gamma_2\trans x)^2-(\gamma_3\trans x)^2]}.
\end{equation}
It requires $\gamma_i$ to be a system of orthogonal unit vectors, with
$\gamma_1$ denoting the mean direction (as $\mu$ in the von Mises-Fisher case),
$\gamma_2$ and $\gamma_3$ describing the semi-major and semi-minor axis of
the covariance. $\kappa$ and $\beta$ describe the concentration and
eccentricity of the covariance.

Though being mathematically more profound, these distributions have
the big disadvantage that they can not be easily combined with other
normal distributions. Presumably, a completely new estimation
algorithm would be needed to compensate for this, whereas our approach
allows us to adapt established algorithms.


 \section{The \mplus-Method}
\label{sec:BoxplusMethod}

In this paper, we propose a method, the \bpm, which integrates generic
sensor fusion algorithms with sophisticated state representations by
encapsulating the state representation structure in an operator \mplus.

In the context of sensor fusion, a \emph{state} refers to those
aspects of reality that are to be estimated. The state comprises the
quantities needed as output for a specific application, such as the
guidance of an aircraft.  Often, additional quantities are needed in
the state in order to model the behavior of the system by mathematical
equations, \eg velocity as the change in position.
\begin{figure}
  \centering
  \inputTikZ{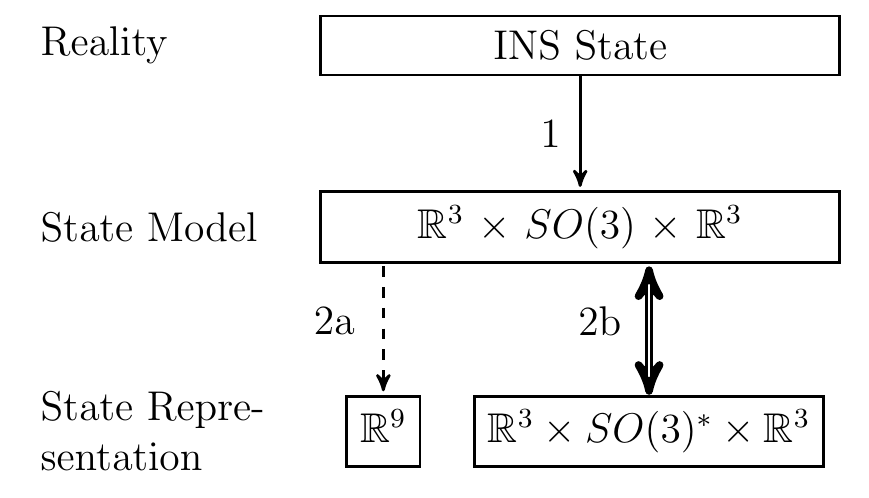}
  \caption{From \emph{reality} to \emph{state representation} in the
    INS example (see text for details).  Step 1 models relevant
    aspects of reality in the \emph{state model}.  Step 2a uses the
    standard approach to obtain a \emph{lossy state representation}.
    Step 2b uses the \bpm to obtain a \emph{state representation
      preserving the topological structure} of the state
    space (using a faithful representation $\SO3^*$ of $\SO3$).} 
  \label{fig:reality-to-state-rep}
\end{figure}

The first step in the design of a sensor fusion system (\cf
Figure~\ref{fig:reality-to-state-rep}) is to come up with a state
model based on abstract mathematical concepts such as position or
orientation in three-dimensional space.  In the aircraft example, the
INS state would consist of three components: the position ($\R^3$),
orientation (\SO3), and velocity ($\R^3$), each in some Earth-centric
coordinate system, \ie
\begin{equation}
\S = \R^3\times\SO3\times\R^3.
\end{equation}
%

\subsection{The Standard Approach to State Representation in Generic Sensor Fusion Algorithms}

The second step (\cf Figure~\ref{fig:reality-to-state-rep}) is to turn
the abstract state model into a concrete representation which is
suitable for the sensor fusion algorithm to work with.  Standard
generic sensor fusion algorithms require the state representation to
be $\R^n$.  Thus, one needs to translate the state model into an
$\R^n$ representation.  In the INS example, choosing Euler angles as
the representation of \SO3 we obtain the following translation:
\begin{equation}
\R^3\times\SO3\times\R^3 \to \R^9
\end{equation}
This translation step \emph{loses information} in two ways: Firstly
and most importantly, it forgets the mathematical structure of the
state space -- as discussed in the introduction, the Euler-angle
representation of three-dimensional orientations has singularities and
there are multiple parameterizations which correspond to the same
orientation.  This is particularly problematic, because the generic
sensor fusion algorithm will exhibit erratic behavior near the
singularities of this representation. If, instead a non-minimal
state representation such as quaternions or rotation matrices 
was used, the translation loses no information. However, later
a generic algorithm would give inconsistent results, 
not knowing about the unit and orthonormality constraints,
unless it was modified in a representation-specific way accordingly.

The second issue is that states are now treated as flat vectors with
the only differentiation between different components being the
respective index in a vector.  All entries in the vector look the
same, and for each original component of a state (\eg position,
orientation, velocity) one needs to keep track of the chosen
representation and the correct index. This may seem like a mere
book-keeping issue but in practice tends to be a particularly
cumbersome and error-prone task when implementing process or
measurement models that need to know about these
representation-specifics.

\subsection{Manifolds as State Representations}

We propose to use manifolds as a central tool to solve both issues.
In this section, we regard a manifold as a black box encapsulating a
certain (topological) structure.  In practice it will be a
\emph{subset of $\R^s$ subject to constraints} such as the
orthonormality of a $3 \times 3$ rotation matrix that can be used as a
representation of \SO3.  

As for the representation of states, in the simple case where a state
consists of a single component only, \eg a three-dimensional
orientation (\SO3), we can represent a state as a single manifold.  If
a state consists of multiple different components we can also
represent it as a manifold since the Cartesian product (Section~\ref{sec:form:cartesian-product}) of the
manifolds representing each individual component yields another,
\emph{compound manifold}.  Essentially, we can build sophisticated
compound manifolds starting with a set of \emph{manifold primitives}.
As we will discuss in Section~\ref{sec:software}, this mechanism can
be used as the basis of an object-oriented software toolkit which
automatically generates compound manifold classes from a concise
specification.

\subsection{\mplus-Manifolds}\label{sec:boxplusmethod:generic}
\begin{modified}

The second important property of manifolds is that they are locally
homeomorphic to $\R^n$, \ie informally speaking, we can establish a
bi-directional mapping from a local neighborhood in an $n$-manifold to
$\R^n$.  The \bpm uses two encapsulation operators $\mplus$
(``boxplus'') and $\mmnus$ (``boxminus'') to implement this mapping:
\begin{alignat}2
\mplus_\S &: \S\times\R^n &&\to \S, \\
\mmnus_\S &: \S\times \S  &&\to\R^n. 
\end{alignat}
When clear from the context, the subscript ${}_\S$ is omitted.  The
operation $y = x \mplus \delta$ adds a small perturbation expressed as
a vector $\delta \in \R^n$ to the state $x \in \S$.  Conversely,
$\delta = y \mmnus x$ determines the perturbation vector $\delta$
which yields $y$ when $\mplus$-added to $x$.  Axiomatically, this is
captured by the definition below. We will discuss its properties here intuitively,
formal proofs are given in Appendix \ref{sec:formalization}.
\begin{dfn}[\mplus-Manifold]\label{dfn:boxplusmanifold}
  A \defi{\mplus-manifold} is a quadruple $(\S,\mplus,\mmnus,V)$
  (usually referred to as just $\S$), consisting of a subset
  $\S\subset\R^s$, operators
\begin{alignat}2
\mplus &: \S\times\R^n &&\to \S, \\
\mmnus &: \S\times \S  &&\to\R^n,
\end{alignat}
and an open neighborhood $V\subset\R^n$ of $0$. These data are subject
to the following requirements. To begin, $\delta\mapsto x\mplus\delta$
must be smooth on $\R^n$, and for all $x\in\S$, $y\mapsto y\mmnus x$
must be smooth on $U_x$, where $U_x=x\mplus V$. Moreover, we impose
the following axioms to hold for every $x\in\S$:
\begin{subequations} \label{eq:frm:axioms}
\begin{alignat}{4}
&\quad&
  x &\mplus 0     &&          &&= x   \label{ax:zero}\\*
\forall y\in\S: &&
  x &\mplus    (y &&\mmnus x) &&= y   \label{ax:surj} \\*
\forall\delta\in V: &&
 (x &\mplus\delta)&&\mmnus x  &&= \delta \label{ax:inj}
\end{alignat}
\begin{equation}
\label{ax:triangle}
\forall\delta_1,\delta_2\in\R^n:
\norm{(x\mplus\delta_1)\mmnus(x\mplus\delta_2)}
 \le\norm{\delta_1-\delta_2}.
\end{equation}
\end{subequations}
\end{dfn}

One can show that a \mplus-manifold is indeed a manifold, with
additional structure useful for sensor fusion algorithms.  The
operators $\mplus$ and $\mmnus$ allow a generic algorithm to modify
and compare manifold states as if they were flat vectors without
knowing the internal structure of the manifold, which thus appears as
a black box to the algorithm.

Axiom~\eqref{ax:zero} makes $0$ the neutral element of \mplus.  Axiom
\eqref{ax:surj} ensures that from an element $x$, every other element
$y\in\S$ can be reached via $\mplus$, thus making $\delta\mapsto
x\mplus\delta$ surjective. Axiom \eqref{ax:inj} makes $\delta\mapsto
x\mplus\delta$ injective on $V$, which defines the range of
perturbations for which the parametrization by $\mplus$ is unique.
Obviously, this axiom cannot hold globally in general, since otherwise
we could have used $\R^n$ as a universal state representation in the
first place.  Instead, $\mplus$ and $\mmnus$ create a local vectorized
view of the state space.  Intuitively $x$ is a reference point which
defines the ``center'' of a local neighborhood in the manifold and
thus also the coordinate system of $\delta$ in the part of $\R^n$ onto
which the local neighborhood in the manifold is mapped (\cf
Figure~\ref{fig:sphere}). The role of Axiom~\eqref{ax:triangle} will
be commented on later.

Additionally, we demand that the operators are \emph{smooth} (\ie
sufficiently often differentiable, cf. Appendix
\ref{sec:formalization}) in $\delta$ and $y$ (for $y\in U_x$).  This
makes limits and derivatives of $\delta$ correspond to limits and
derivatives of $x\mplus\delta$, essential for any estimation algorithm
(formally, $\delta\mapsto x\mplus\delta$ is a diffeomorphism from $V$
to $U_x$).  It is important to note here that we require neither
$x\mplus\delta$ nor $y\mmnus x$ to be smooth in $x$. Indeed it is
sometimes impossible for these expressions to be even continuous in
$x$ for all $x$ (see Appendix~\ref{sec:ex:S2}).
Axiom \eqref{ax:triangle} allows to define a metric and is discussed 
later.

Returning to the INS example, we can now essentially keep the state
model as the state representation
(Figure~\ref{fig:reality-to-state-rep}, Step 2b):
\begin{equation*}
\R^3\times\SO3\times\R^3 \to \R^3\times\SO3^*\times\R^3,
\end{equation*}
where $\SO3^*$ refers to any mathematically sound (``lossless'')
representation of $\SO3$ expressed as a set of numbers to enable a
computer to process it.  Commonly used examples would be quaternions
($\R^4$ with unit constraints) or rotation matrices ($\R^{3 \times 3}$
with orthonormality constraints). Additionally, we need to define
matching representation-specific $\mplus$ and $\mmnus$ operators which
replace the static, lossy translation of the state model into an
$\R^n$ state representation that we saw in the standard approach above
with an on-demand, lossless mapping of a manifold
state representation into $\R^n$ in our approach.

In the INS example, $\mplus$ would simply perform vector-arithmetic on
the $\R^n$ components and multiply a small, minimally parameterized
rotation into the $\SO3^*$ component (details follow soon).

\subsection{Probability Distributions on \mplus-Manifolds}
\label{sec:bpm:PDonM}
So far we have developed a new way to represent states -- as compound
manifolds -- and a method that allows generic sensor fusion algorithms
to work with them -- the encapsulation operators \mplus/\mmnus. Both
together form a \mplus-manifold.
However, sensor fusion algorithms rely on the use of probability
distributions to represent uncertain and noisy sensor data.  Thus, we
will now define probability distributions on \mplus-manifolds.

The general idea is to use a manifold element as the mean $\mu$ which
defines a reference point.  A multivariate
probability distribution which is well-defined on $\R^n$ is then
lifted into the manifold by mapping it into the neighborhood around
$\mu \in \S$ via $\mplus$.  That is, for $X:\Omega\to\R^n$ and $\mu\in
\S$ (with $\dim \S=n$), we can define $Y:\Omega\to\S$ as $Y:=\mu\mplus X$, with probability
distribution given by
\begin{equation}\label{eq:bpm:DistOnManiDef}
\p(Y=y) = \p(\mu\mplus X=y) 
\end{equation}
In particular, we extend the notion of a Gaussian distribution to \mplus-manifolds by
\begin{equation}\label{eq:bpm:gaussian}
    \Nd(\mu, \Sigma) := \mu \mplus \Nd(0, \Sigma),
\end{equation}
where $\mu\in\S$ is an element of the \mplus-manifold but $\Sigma\in\R^{n\times n}$
just a matrix as for regular Gaussians (App.\@~\ref{sec:form:PDonM}).

\subsection{Mean and Covariance on \mplus-Manifolds}
\label{sec:bpm:meanandcov}

Defining the expected value on a manifold is slightly more involved
than one might assume: we would, of course, expect that $\E X\in\S$
for $X:\Omega\to\S$, which however would fail for a naive definition such as
\begin{equation}
\E X \stackrel?= \int_\S x\cdot\p(X=x)\dx.
\end{equation}
Instead, we need a definition that is equivalent to the definition on $\R^n$
and well defined for \mplus-manifolds. Therefore, we define the expected value as the value minimizing the
expected mean squared error:
\begin{equation}\label{eq:bpm:E-def}
\E X = \argmin_{x\in\S} \E(\norm{X \mmnus x}^2)
\end{equation}
This also implies the implicit definition
\begin{equation}\label{eq:bpm:E-implicit}
\E(X\mmnus\E X) = 0,
\end{equation}
as we will prove in Appendix~\ref{sec:form:PDonM}. 

One method to compute this value is to start with an initial guess
$\mu_0$ and iterate \cite{Kraft2003}:
\begin{align}
 \mu_{k+1} &= \mu_k\mplus\E(X\mmnus\mu_k)\label{eq:boxplusmethod:expectedval1}\\
 \E X      &= \lim_{k\to\infty}\mu_k.\label{eq:boxplusmethod:expectedval2}
\end{align}
Care must be taken that $\mu_0$ is
sufficiently close to the true expected value. In practice, however,
this is usually not a problem as sensor fusion algorithms typically
modify probability distributions only slightly at each time step such
that the previous mean can be chosen as $\mu_0$.  

Also closed form solutions exist for some manifolds -- most trivially
for $\R^n$. For rotation matrices, Markley et al.\@~\cite{Markley2007}
give a definition similar to \eqref{eq:bpm:E-def} but use Frobenius
distance. They derive an analytical solution. Lemma
\ref{lem:mat:frobenius} shows that both definitions are roughly
equivalent, so this can be used to compute an initial guess for
rotations.

The same method can be applied to calculate a weighted mean value of a number of
values, because in vector spaces the weighted mean
\begin{equation}
\bar x = \sum_i w_ix_i\quad\text{with}\quad \sum_i w_i=1
\end{equation}
can be seen as the expected value of a discrete distribution with
$\P(X=x_i)=w_i$.

\begin{figure}
\centering
  \inputTikZ{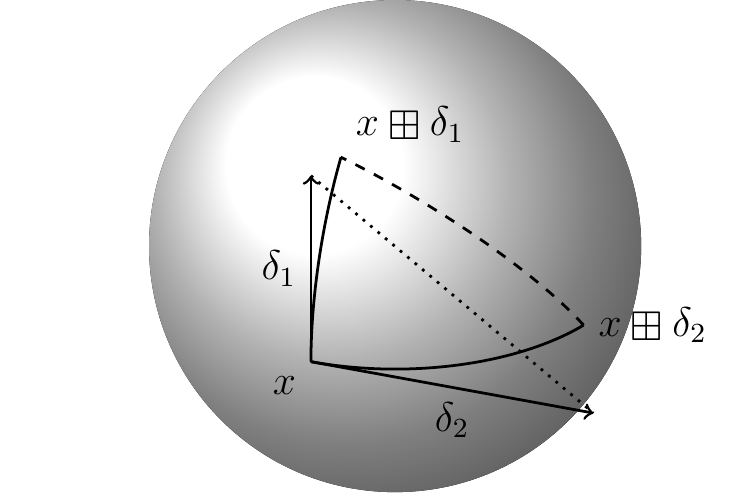}
  \caption{Axiom \eqref{ax:triangle}: The $\d$-distance between $x\mplus\delta_1$ and $x\mplus\delta_2$
  (dashed line) is less or equal to the distance in the parameterized around
  $x$ (dotted line).}
  \label{fig:form:triangle}
\end{figure}

The definition of the covariance of a \mplus-manifold distribution, on the
other hand, is straightforward.  As in the $\R^n$ case, it is an $n
\times n$ matrix.  Specifically, given a mean value $\E X$ of a
distribution $X:\Omega\to \S$, we define its covariance as
\begin{equation}\label{eq:bpm:covariance}
\Cov X = \E\left(\outerProd{(X\mmnus\E X)}\right),
\end{equation}
because $X\mmnus\E X\in\R^n$ and the standard definition can be applied.

The $\mmnus$-operator induces a metric (proof in App.\@~\ref{sec:formalization})
\begin{equation}\label{eq:form:metric-dfnA}
\d(x,y):=\norm{y\mmnus x}
\end{equation}
that is important in interpreting $\Cov X$.
First, 
\begin{align}
    \trace\Cov X=\E\left( \d(X,\E X)^2\right).
\end{align}
\ie $\sqrt{\trace\Cov X}$ is the rms $\d$-distance of $X$ to the mean.
Second, the states $y\in\S$
with $\d(\mu,y)=\sigma$ are the $1\sigma$ contour of $\Nd(\mu,\sigma^2\I)$. Hence,
to interpret the state covariance or to define measurement or dynamic noise for a sensor fusion algorithm intuitively it
is important that the metric induced by \mmnus has an intuitive meaning. For the
\mplus-manifold representing orientation in the INS example, $\d(x,y)$
will be the angle between two orientations $x$ and $y$.

In the light of \eqref{eq:form:metric-dfnA}, axiom \eqref{ax:triangle} means 
that the actual distance $\d(x\mplus\delta_1, x\mplus\delta_2)$ is less or equal
to the distance $\|\delta_1-\delta_2\|$ in the parametrization (Fig.\@~\ref{fig:form:triangle}), i.e.\ the map $x\mapsto x\mplus\delta$ is $1$-Lipschitz. This axiom is needed
for all concepts explained in this subsection, as can be seen in
the proofs in Appendix \ref{sec:formalization}. In so far, \eqref{ax:triangle}
is the deepest insight in the axiomatization \eqref{eq:frm:axioms}.

Appendix \ref{sec:formalization} also
discusses a slight inconsistency in the definitions above, where
the covariance of $\Nd(\mu, \Sigma)$ defined by
\eqref{eq:bpm:gaussian} and \eqref{eq:bpm:covariance} is slightly
smaller than $\Sigma$. For the usual case that $\Sigma$ is significantly smaller than the
range of unique parameters $V$ (\ie for angles $\ll\pi$), these
inconsistencies are very small and can be practically ignored.

\subsection{Practically Important \mplus-Manifolds}
\label{sec:bpm:orientation}

We will now define \mplus-operators for the practically most relevant
manifolds $\R^n$ as well as 2D and 3D orientations. 
Appendix \ref{sec:examplemanifolds} proves the ones listed here to fulfill the
axioms \eqref{eq:frm:axioms}, Appendices \ref{sec:form:symmetric-encaps} and \ref{sec:liegroups}
derive general techniques for constructing a \mplus-operator.

\subsubsection{Vectorspace (\/\Rn\/)} 

For $\S=\R^n$, the \mplus- and \mmnus-operators are, of course, simple vector
addition and subtraction
\begin{align}
    x\mplus\delta = x+\delta, \quad y\mmnus{}x=y-x.
    \label{eq:boxplus:R}
\end{align}

\subsubsection{2D Orientation as Angles Modulo \texorpdfstring{$2\pi$}{2pi}} 
\label{sec:bpm:twodrotation}

A planar rotation is the simplest case that requires taking its manifold
structure into account. It can be represented by the rotation angle 
interpreted modulo $2\pi$. Mathematically, this is $\S=\R/2\pi\Z$,
a set of equivalence classes. Practically, simply a real number $\alpha\in\R$ is stored, and
periodic equivalents $\alpha+2\pi{}k, k\in\Z$ are treated as the same.
\mplus is, then, simply plus. It
could be normalized but that is not necessary.
The difference
in \mmnus however, being a plain real value, must be normalized to $[-\pi,\pi)$ using a 
function $\nu_\pi$:
\begin{align}
    \begin{split}
    \alpha\mplus\delta = \alpha+\delta, \quad \beta\mmnus\alpha = \nu_\pi(\beta-\alpha),\\
    \text{where } \nu_\pi(\delta):=\delta-2\pi\floor{\tfrac{\delta+\pi}{2\pi}}
    \end{split}
\label{eq:boxplus:angle}
\end{align}
With this definition, $\beta\mmnus\alpha$ is the smallest angle needed
to rotate $\alpha$ into $\beta$, respecting the periodic interpretation of angles
and giving the induced metric $\d(x,y)$ an intuitive meaning. 
The parametrization is unique for angles of modulus $<\pi$, \ie $V=(-\pi,\pi)$.

\subsubsection{3D Orientation as an Orthonormal Matrix}\label{sec:bpm:3dMatrix}
Rotations in 3D can be readily represented using orthonormal matrices
$\S\subset\R^{3\times3}$ with determinant $1$. $x\mplus\delta$ performs a
rotation around axis $\delta$ in $x$ coordinates with angle $\norm{\delta}$. This is
also called matrix exponential representation and implemented
by the Rodriguez formula~\cite[pp.\,147]{Koks06}
\begin{gather}
   x\mplus{}\delta = x \exp \delta, \quad y\mmnus{}x = \log\left(x^{-1}y\right),\\
   \begin{split}
   \exp\matrxs{ x\\y\\z } = 
  \matrxs{ 
         \cos\theta+cx^2 & -sz+cxy & sy+cxz \\
         sz+cxy & \cos\theta+cy^2 & -sx+cyz \\
         -sy+cxz & sx+cyz & \cos\theta + cz^2
  }, \\
  \theta = \sqrt{x^2+y^2+z^2}, s=\sinc\theta, c = \tfrac{1-\cos\theta}{\theta^2}
   \end{split}
\label{eq:exp:mat} \\%
   \log x =  \frac{\theta}{2\sin\theta} \; \matrxs{ x_{32}-x_{23}\\x_{13}-x_{31}\\x_{21}-x_{12} },
   \theta = \arccos\tfrac{\trace x-1}{2}.
   \label{eq:log:mat}
\end{gather}
The induced metric $\d(x,y)$ is the angle of a rotation necessary to
rotate $x$ onto $y$.  There is also a monotonic relation to the widely
used Frobenius distance $\normF{x-y}$ (Lemma \ref{lem:mat:frobenius},
App.\@~\ref{sec:technicalProofs}).  The parametrization is again
unique for angles $<\pi$, \ie $V=B_\pi(0)$, where, as usual, we denote
by
\begin{equation}
  \label{eq:openball}
  B_\eps(v)=\setwith{w\in\R^n}{\norm{w-v}<\eps}
\end{equation}
the open $\eps$-ball around $v\in\R^n$ for $\eps>0$.

\subsubsection{3D Orientation as a Unit Quaternion}\label{sec:bpm:quaternion} The same geometrical construction of rotating 
around $\delta$ by $\norm\delta$ also works with unit quaternions
$\S\subset\mathbb{H}$, where $q$ and $-q$ are considered equivalent as they represent the same orientation.
\begin{align}
q\mplus\delta &=q\cdot\exp\tfrac\delta2, \;\;
q \mmnus p =2\olog(p\inv\cdot q),
\label{eq:quaternionboxplus:def} \\
    \exp\delta&=\quat{\cos\norm\delta}{\sinc\norm\delta \delta},
    \label{eq:exp:quat} \\
   \olog\quat wv &= 
   \begin{cases} 
      0 & v=0\\
      \tfrac{\atan(\norm v / w)}{\norm v}v & v\ne0, w\ne0\\
      \pm\tfrac{\pi/2}{\norm v} v & w=0.
    \end{cases}
    \label{eq:log:quat}    
\end{align}
The factor $2$ is introduced so that the induced metric $\d(p,q)$ is the
angle between two orientations and the quaternion and matrix
\mplus-manifolds are isomorphic. It originates from the fact that the
quaternion is multiplied to the left and right of a vector when
applying a rotation. The equivalence of $\pm q$ causes
the $\atan(\norm v / w)$ term in \eqref{eq:log:quat} instead of $\atantwo(\norm v, w)$,
making $\olog q=\olog(-q)$.

\subsubsection{Compound \mplus-Manifolds}
\label{sec:compoundmanifolds}

Two (or several) \mplus-manifolds $\S_1$, $\S_2$ can be combined into a single \mplus-manifold
by simply taking the Cartesian product $\S=\S_1\times\S_2$ and defining the \mplus- and
\mmnus-operator component-wise as
\begin{align}
(x_1, x_2)\mplus\matrxs{\delta_1\\ \delta_2} 
  &:= (x_1\mplus_{\S_1}\delta_1, x_2\mplus_{\S_2}\delta_2) \\
(y_1, y_2)\mmnus(x_1, x_2)
  &:=\matrxs{y_1\mmnus_{\S_1} x_1\\ y_2\mmnus_{\S_2} x_2}.
  \label{eqn:compoundmanifolds}
\end{align}

\end{modified}


\section{Least Squares Optimization and Kalman Filtering on
  \mplus-Manifolds}\label{sec:application}

\newcommand*\gxiZ{f_{x_i}^z}
\newcommand*\gxiZf[1]{f(x_i\mplus#1)\mmnus z}
\begin{table*}
\begin{align}
\intertext{\textbf{\hfil Classical Gauss-Newton\hfil
\hfil Gauss-Newton on a \mplus-Manifold\hfil}}
&f:\R^n\to\R^m &
&f:\S\to\M \\
f(X)&\sim\Ndp z\Sigma\eqv f(X) - z \sim\Ndp0\Sigma &
f(X)&\sim z\mplus\Ndp0\Sigma \stackrel*\eqv f(X)\mmnus z\sim\Ndp0\Sigma \\
\intertext{\centering Iterate with initial guess $x_0$ until $x_i$ converges:}
J_{\bullet k} &:= \frac{f(x_i+\eps e_k) - f(x_i-\eps e_k)}{2\eps} &
J_{\bullet k} &:= \frac{(\gxiZf{\eps e_k}) - (\gxiZf{-\eps e_k})}{2\eps} 
\label{eq:ls:jacobian}\\
x_{i+1} &:= x_i - (J\trans\Sigma\inv J)\inv J\trans\Sigma\inv(f(x_i)-z) &
x_{i+1} &:= x_i \mplus - (J\trans\Sigma\inv J)\inv J\trans\Sigma\inv(f(x_i)\mmnus z)
\label{eq:ls:update}
\end{align}
  \caption{Only small changes are necessary to adapt a classical least squares
    algorithm (left column) to work on \mplus-manifolds (right column).
    Adding perturbations to the state is done using \mplus, comparing values in
    the measurement space is done using \mmnus.
    Note that the $\mmnus z$ term does not cancel out when calculating the
    Jacobian. Also note that the equivalence marked by ${}^*$ holds only
    approximately, as we will derive in Appendix~\ref{sec:form:PDonM}.}
\label{tab:algo-LS}
\end{table*}

One of the goals of the \bpm was to easily adapt estimation algorithms
to work on arbitrary manifolds.  Essentially, this can be done by
replacing $+$ with $\mplus$ when adding perturbations to a state, and
replacing $-$ with $\mmnus$ when comparing two states or measurements.
However, some pitfalls may arise, which we will deal with in this
section.

We show how the \bpm can be applied to convert least
squares optimization algorithms and the Unscented Kalman Filter such
that they can operate on \mplus-manifolds rather than just $\R^n$.

\subsection{Least Squares Optimization}

Least squares optimization dates back to the late 18th century where
it was used to combine measurements in astronomy and in geodesy.
Initial publications were made by Legendre \cite{Legendre1805} and
Gauss \cite{Gauss1821} in the early 19th century.
The method is commonly used to solve overdetermined problems, \ie
problems having more ``equations'' or measurements than unknown
variables.

When combining all unknowns into a single state $X\in\R^n$, and all
measurement functions into a single function $f:\R^n\to\R^m$, the
basic idea is to find $X$ such that given a combined measurement~$z$,
\begin{equation}
\tfrac12\norm{f(X)-z}^2=\min!
\end{equation}
(where we write ``$=\min!$'' to denote that the left-hand side becomes
minimal).  For a positive definite covariance $\Sigma$
between the measurements, this becomes
\begin{equation}\label{eq:lsproblem-classic}
  \tfrac12\norm{f(X)-z}_\Sigma^2=\min!
\end{equation}
using the notation $\norm x^2_\Sigma := x\trans\Sigma\inv
x$.  Under the assumption that
$f(X)=z+\eps$, with $\eps\sim\Ndp0\Sigma$, this leads to a maximum
likelihood solution.

If now our measurement function $f:\S\to\M$ maps from a state manifold $\S$ to
a measurement manifold $\M$, we can write analogously:
\begin{equation}\label{eq:lsproblem-manifold}
\tfrac12\norm{f(X)\mmnus z}_\Sigma^2 = \min!
\end{equation}
which again leads to a maximum likelihood solution,
for $f(X)\mmnus z\sim\Ndp0\Sigma$, as we will prove in Appendix~\ref{sec:form:PDonM}.

Even for classical least squares problems, nonlinear functions $f$
usually allow only local, iterative solutions, \ie starting from an
initial guess $x_0$ we construct a sequence of approximations $x_i$ by
calculating a refinement $\delta_i$ such that $x_{i+1}=x_i+\delta_i$
is a better solution than $x_i$.

This approach can be
adapted to the manifold case, where every iteration takes place on a
new local function
\begin{align}
\gxiZ:\R^n&\to\R^m \\
    \delta&\mapsto \gxiZf\delta,
\end{align}
which for smooth $f$ is a smooth function and, as it is an ordinary
vector function, a local refinement $\delta_i$ can be found
analogously to the classical case.  This refinement can then be added
to the previous state using $x_{i+1}:=x_i\mplus\delta_i$.  The key
difference is that now the refinements are accumulated in $\S$, not in $\R^n$.
In Table~\ref{tab:algo-LS} we show how this can be done using the popular
Gauss-Newton method with finite-difference Jacobian calculation.
Other least squares methods like Levenberg-Marquardt can be applied analogously.

\subsection{Kalman Filtering}
Since its inception in the late 1950s, the Kalman filter
\cite{Kalman1960} and its many variants have successfully been applied
to a wide variety of state estimation and control problems.  In its
original form, the Kalman filter provides a framework for continuous
state and discrete time state estimation of linear Gaussian systems.
Many real-world problems, however, are intrinsically non-linear, which
gives rise to the idea of modifying the Kalman filter algorithm to
work with non-linear process models (mapping old to new state) and
measurement models (mapping state to expected measurements) as well.

The two most popular extensions of this kind are the Extended Kalman
Filter (EKF) \cite[Chap.\,5.2]{BarShalom2001} and more recently the
Unscented Kalman Filter (UKF) \cite{Julier97anew}.  The EKF linearizes
the process and measurement models through first order Taylor series
expansion.  The UKF, on the other hand, is based on the \ut which
approximates the respective probability distributions through
deterministically chosen samples, so-called \sips, propagates these directly
through the non-linear process and measurement models and recovers the
statistics of the transformed distribution from the transformed
samples.  Thus, intuitively, the EKF relates to the UKF as a tangent to a
secant.

We will focus on the UKF here since it is generally better at handling
non-linearities and does not require (explicit or numerically
approximated) Jacobians of the process and measurement models, \ie it
is a \emph{derivative-free} filter.  Although the UKF is fairly new,
it has been used successfully in a variety of robotics applications
ranging from ground robots \cite{Thrun06} to unmanned aerial vehicles
(UAVs) \cite{Merwe2004SPKF}.

The UKF algorithm has undergone an evolution from early publications
on the \ut \cite{Quine95implicit} to the work by Julier and Uhlmann
\cite{Julier96ageneral,Julier97anew} and by van der Merwe et al.\@
\cite{rvdmerwe_waml2003,Merwe2004SPKF}. The following is based on the
consolidated UKF formulation by Thrun, Burgard \& Fox \cite{tbf05} with parameters chosen as
discussed in \cite{Wagner2010}.  The modification of the UKF algorithm
for use with manifolds is based on \cite{tdsf06}, \cite{Kurlbaum2007},
\cite{Birbach2008} and \cite{Wagner2010}.

\subsubsection{Non-Linear Process and Measurement Models}\label{sec:kalman:processandmeasurementmodels}
UKF process and measurement models need not be linear but are assumed
to be subject to additive white Gaussian noise, \ie
\begin{align}
x_t & = g(u_t,x_{t-1}) + \eps_t\\
z_t & = h(x_t) + \delta_t
\end{align}
where $g:T \times \R^n\to \R^n$ and $h:\R^n\to \R^m$ are arbitrary
(but sufficiently nice) functions, $T$ is the space of controls,
$\eps_t \sim \Ndp{0}{R_t}$, $\delta_t \sim \Ndp{0}{Q_t}$, and all
$\eps_t$ and $\delta_t$ are independent.

\subsubsection{Sigma Points}

The set of $2n+1$ \sips that are used to approximate an
$n$-dimensional Gaussian distribution with mean $\mu$ and covariance
$\Sigma$ is computed as follows:
\begin{align}
\sigp{X}{0}   & =  \mu \\
\sigp{X}{i}   & =  \mu + \bigl(\sqrt{\Sigma}\bigr)_{\bullet i} \text{ for } i = 1,\ldots,n\\
\sigp{X}{i+n} & =  \mu - \bigl(\sqrt{\Sigma}\bigr)_{\bullet i} \text{ for } i = 1,\ldots,n
\end{align}
where $\bigl(\sqrt{\Sigma}\bigr)_{\bullet i}$ denotes the $i$-th
column of a matrix square root $\sqrt{\Sigma}\sqrt{\Sigma}^T=\Sigma$  
implemented by Cholesky decomposition.  The name \sips
reflects the fact that all \sigp{X}{k} lie on the $1\sigma$-contour
for $k>0$.

In the following we will use the abbreviated notation
\begin{equation}
\mathcal{X} = (\mu \hspace{1cm} \mu+\sqrt{\Sigma} \hspace{1cm} \mu-\sqrt{\Sigma})
\end{equation}
to describe the generation of the \sips.

\subsubsection{Modifying the UKF Algorithm for Use with \mplus-Manifolds}

The complete UKF algorithm is given in
Table~\ref{tab:algo-UKF}.
\newcommand\sigmaspace{\hspace{0.25cm}}
\begin{table*}[p]
\begin{align}
\intertext{\textbf{\hfil Classical UKF\hfil
\hfil UKF on \mplus-Manifolds\hfil}}\nonumber\\[-1.8\baselineskip]
\intertext{\centering Input, Process and Measurement Models:}\nonumber\\[-1.5\baselineskip]
%
\mu_{t-1} &\in \R^n, \Sigma_{t-1} \in \R^{n \times n}, u_t \in T, z_t \in \R^m&
\mu_{t-1} &\in \S, \Sigma_{t-1} \in \R^{n \times n}, u_t \in T, z_t \in \M\\
%
g &: T \!\times\! \R^n \to\! \R^n, X_t \!=\! g(u_t, X_{t-1})+\Nd(0,R_t)& 
g &: T \!\times\! \S \to\! \S, X_t \!=\! g(u_t, X_{t-1})\mplus\Nd(0,R_t)\\
h &: \R^n\to\R^m, z_t = h(X_t) + \Nd(0,Q_t)&
h &: \S\to\M, z_t = h(X_t) \mplus_\M \Nd(0,Q_t)\\
\intertext{\centering Prediction Step:}\nonumber\\[-1.5\baselineskip]
\mathcal{X}_{t-1} &= (\mu_{t-1} \sigmaspace \mu_{t-1}+\sqrt{\Sigma_{t-1}} \sigmaspace \mu_{t-1}-\sqrt{\Sigma_{t-1}})&
\mathcal{X}_{t-1} &= (\mu_{t-1} \sigmaspace \mu_{t-1}\mplus\sqrt{\Sigma_{t-1}} \sigmaspace \mu_{t-1}\mplus-\sqrt{\Sigma_{t-1}})\label{eq:ukf:predict:gensigma}\\
\bar{\mathcal{X}}_t^* &= g(u_t, \mathcal{X}_{t-1})&
\bar{\mathcal{X}}_t^* &= g(u_t, \mathcal{X}_{t-1})\label{eq:ukf:predict:transsigma}\\
\bar\mu_t &= \frac{1}{2n+1} \sum_{i=0}^{2n}\bar{\mathcal{X}}_t^{*[i]}&
\bar\mu_t &= \meanofsigmap(\bar{\mathcal{X}}_t^*)\label{eq:ukf:predict:recmean}\\
\bar\Sigma_t &= \frac{1}{2} \sum_{i=0}^{2n}(\bar{\mathcal{X}}_t^{*[i]} - \bar\mu_t)(\bar{\mathcal{X}}_t^{*[i]} - \bar\mu_t)\trans+R_t&
\bar\Sigma_t &= \frac{1}{2} \sum_{i=0}^{2n}(\bar{\mathcal{X}}_t^{*[i]} \mmnus \bar\mu_t)(\bar{\mathcal{X}}_t^{*[i]} \mmnus \bar\mu_t)\trans+R_t\label{eq:ukf:predict:recsigma}\\
\intertext{\centering Correction Step:}\nonumber\\[-1.5\baselineskip]
\bar{\mathcal{X}}_{t} &= (\bar\mu_{t} \hspace{.5cm} \bar\mu_{t}+\sqrt{\bar\Sigma_{t}} \hspace{.5cm} \bar\mu_{t}-\sqrt{\bar\Sigma_{t}})&
\bar{\mathcal{X}}_{t} &= (\bar\mu_{t} \hspace{.5cm} \bar\mu_{t}\mplus\sqrt{\bar\Sigma_{t}} \hspace{.5cm} \bar\mu_{t}\mplus-\sqrt{\bar\Sigma_{t}})\label{eq:ukf:correct:genxsigma}\\
\bar{\mathcal{Z}}_t &= h(\bar{\mathcal{X}}_{t})&
\bar{\mathcal{Z}}_t &= h(\bar{\mathcal{X}}_{t})\label{eq:ukf:correct:zsigma}\\
\hat z_t &= \frac{1}{2n+1}\sum_{i=0}^{2n}\bar{\mathcal{Z}}_t^{[i]}&
\hat z_t &= \meanofsigmap(\bar{\mathcal{Z}}_t)\label{eq:ukf:correct:meanz}\\
S_t &= \frac{1}{2}\sum_{i=0}^{2n}(\bar{\mathcal{Z}}_t^{[i]} - \hat z_t)(\bar{\mathcal{Z}}_t^{[i]} - \hat z_t)\trans+Q_t&
S_t &= \frac{1}{2}\sum_{i=0}^{2n}(\bar{\mathcal{Z}}_t^{[i]} \mmnus_\M \hat z_t)(\bar{\mathcal{Z}}_t^{[i]} \mmnus_\M \hat z_t)\trans+Q_t\label{eq:ukf:correct:S}\\
\bar{\Sigma}_t^{x,z} &= \frac{1}{2}\sum_{i=0}^{2n}(\bar{\mathcal{X}}_t^{[i]} - \bar\mu_t)(\bar{\mathcal{Z}}_t^{[i]} - \hat z_t)\trans&
\bar{\Sigma}_t^{x,z} &= \frac{1}{2}\sum_{i=0}^{2n}(\bar{\mathcal{X}}_t^{[i]} \mmnus \bar\mu_t)(\bar{\mathcal{Z}}_t^{[i]} \mmnus_\M \hat z_t)\trans\label{eq:ukf:correct:crosscov}\\
K_t &= \bar{\Sigma}_t^{x,z}S_t^{-1}&
K_t &= \bar{\Sigma}_t^{x,z}S_t^{-1}\label{eq:ukf:correct:K}\\
\mu_t &= \bar\mu_t + K_t(z_t - \hat z_t)&
\delta &= K_t(z_t \mmnus_\M \hat z_t)\label{eq:ukf:correct:newmean}\\
\Sigma_t &= \bar\Sigma_t - K_t S_t K_t\trans&
\Sigma_t' &= \bar\Sigma_t - K_t S_t K_t\trans\label{eq:ukf:correct:newsigma}\\
&&
\mathcal{X}_t' &= (\bar\mu_{t}\mplus\delta \sigmaspace \bar\mu_{t}\mplus(\delta+\sqrt{\Sigma_{t}'}) \sigmaspace \bar\mu_{t}\mplus(\delta-\sqrt{\Sigma_{t}'}))\label{eq:ukf:correct:extrasigmap:gen}\\
&&
\mu_t &= \meanofsigmap(\mathcal{X}_t')\label{eq:ukf:correct:extrasigmap:mean}\\
&&
\Sigma_t &= \frac{1}{2} \sum_{i=0}^{2n}(\mathcal{X}_t'^{[i]} \mmnus \mu_t)(\mathcal{X}_t'^{[i]} \mmnus \mu_t)\trans\label{eq:ukf:correct:extrasigmap:sigma}
\end{align}
\caption{Classical UKF vs. UKF on \mplus-manifolds algorithms. See text for details.}
\label{tab:algo-UKF}
\end{table*}

%
Like other Bayes filter instances, it consists of two alternating
steps -- the prediction and the correction step.  The prediction step
of the UKF takes the previous belief represented by its mean
$\mu_{t-1}$ and covariance $\Sigma_{t-1}$ and a control $u_t$ as
input, calculates the corresponding set of \sips, applies the process
model to each \emph{sigma point}, and recovers the statistics of the
transformed distribution as the predicted belief with added process
noise $R_t$ (\eqref{eq:ukf:predict:gensigma} to
\eqref{eq:ukf:predict:recsigma}).

To convert the prediction step of the UKF for use with
\mplus-manifolds we need to consider operations that deal with states.
These add a perturbation vector to a state in
\eqref{eq:ukf:predict:gensigma}, determine the difference between two
states in \eqref{eq:ukf:predict:recsigma}, and calculate the mean of a
set of sigma points in \eqref{eq:ukf:predict:recmean}.  In the
manifold case, perturbation vectors are added via \mplus and the
difference between two states is simply determined via \mmnus.  The
mean of a set of manifold \sips can be computed analogously to the
definition of the expected value from
\eqref{eq:boxplusmethod:expectedval1} and
\eqref{eq:boxplusmethod:expectedval2}; the definition of the
corresponding function \meanofsigmap is shown in
Table~\ref{tab:algo-meanofsigmap}.
\begin{table}
\begin{align}
\intertext{\textbf{\hfil \mplus-Manifold-\meanofsigmap\hfil}}\nonumber\\[-1.8\baselineskip]
\intertext{\centering Input:}
\sigp{Y}{i},& \;\; i = 0,\ldots,2n\\
\intertext{\centering Determine mean $\mu'$:}
\mu_0' &= \sigp{Y}{0}\\
\mu_{k+1}' &= \mu_k' \mplus \frac{1}{2n+1} \sum_{i=0}^{2n}\sigp{Y}{i} \mmnus \mu_k'\\
\mu' &= \lim_{k \rightarrow \infty}\mu_k'\label{eq:meanofsigmap:limit}
\end{align}
\caption{\meanofsigmap computes the mean of a set of \mplus-manifold sigma points $\mathcal{Y}$.  In practice, the limit in \eqref{eq:meanofsigmap:limit} can be implemented by  an iterative
  loop that is terminated if the norm of the most recent summed error
  vector is below a certain threshold. The number of \emph{sigma points} is not necessarily the same as the dimension of each.}
\label{tab:algo-meanofsigmap}
\end{table}


The UKF correction step first calculates the new set of \sips
\eqref{eq:ukf:correct:genxsigma}, propagates each through the
measurement model to obtain the \sips corresponding to the expected
measurement distribution in \eqref{eq:ukf:correct:zsigma}, and
recovers its mean $\hat z_t$ in \eqref{eq:ukf:correct:meanz} and
covariance $S_t$ with added measurement noise $Q_t$ in
\eqref{eq:ukf:correct:S}.  Similarly, the cross-covariance
$\bar{\Sigma}_t^{x,z}$ between state and expected measurement is
calculated in \eqref{eq:ukf:correct:crosscov}.  The latter two are
then used in \eqref{eq:ukf:correct:K} to compute the Kalman gain $K$,
which determines how the innovation is to be used to update the mean
in \eqref{eq:ukf:correct:newmean} and how much uncertainty can be
removed from the covariance matrix in \eqref{eq:ukf:correct:newsigma}
to reflect the information gained from the measurement.

The conversion of the UKF correction step for use with \mplus-manifolds
generally follows the same strategy as that of the prediction step but
is more involved in detail.  Firstly,
this is because we use \mplus-manifolds to represent both states and
measurements so that the advantages introduced for states above also
apply to measurements.  Secondly, the update of the mean in
\eqref{eq:ukf:correct:newmean} cannot be implemented as a simple
application of \mplus: This might result in an inconsistency between
the covariance matrix and the mean since in general
\begin{equation}
\mu\mplus\Ndp{\delta}{\Sigma'}
\ne (\mu\mplus\delta) \mplus \Ndp{0}{\Sigma'},
\end{equation}
\ie the mean would be modified while the covariance is still
formulated in terms of the coordinate system defined by the old mean
as the reference point.  Thus, we need to apply an additional \sip
propagation as follows.  The manifold variant of
\eqref{eq:ukf:correct:newmean} only determines the perturbation vector
$\delta$ by which the mean is to be changed and the manifold variant
of \eqref{eq:ukf:correct:newsigma} calculates a temporary covariance
matrix $\Sigma_t'$ still relative to the old mean $\bar\mu_t$.
\eqref{eq:ukf:correct:extrasigmap:gen} then adds the sum of $\delta$
and the respective columns of $\Sigma_t'$ to $\bar\mu_t$ in a single
$\mplus$ operation to generate the set of \sips $\mathcal{X}_t'$.
Therefrom the new mean $\mu_t$ in
\eqref{eq:ukf:correct:extrasigmap:mean} and covariance $\Sigma_t$ in
\eqref{eq:ukf:correct:extrasigmap:sigma} is computed.

The overhead of the additional \sip propagation can be avoided by storing $\mathcal{X}_t'$
for reuse in \eqref{eq:ukf:predict:gensigma} or \eqref{eq:ukf:correct:genxsigma}. If $x\mplus\delta$
is continuous in $x$, the step can also be replaced by $\mu_t=\bar\mu_t \mplus \delta$ as an approximation.

\begin{modified}
    A final word auf caution: \emph{Sigma point} propagation fails for a standard deviation
    larger than the range $V$ of unique parametrization, where even propagation
    through the identity function results in a reduced covariance. To prevent this, the standard deviation
    must be within $V/2$, so all \emph{sigma points} are mutually within a range
    of $V$. For 2D and 3D orientation 
    hence an angular standard deviation of $\pi/2$ is allowed. This is no practical
    limitation, because filters usually fail much earlier because of nonlinearity.
\end{modified}



\section{\mplus-Manifolds as a \\Software Engineering Tool}\label{sec:software}
As discussed in Section~\ref{sec:BoxplusMethod}, the \bpm
simultaneously provides two alternative views of a \mplus-manifold.
On the one hand, generic algorithms access \emph{primitive} or
\emph{compound manifolds} via flattened perturbation vectors, on the
other hand, the user implementing process and measurement models needs
direct access to the underlying state representation (such as a
quaternion) and for \emph{compound manifolds} wants to access
individual components by a descriptive name.

In this section we will use our Manifold Toolkit~(\emph{MTK}) to
illustrate how the \bpm can be modeled in software.  The
current version of \emph{MTK} is implemented in C++ and uses the Boost
Preprocessor library \cite{BOOST:PP}
\begin{modified}
and the Eigen Matrix library \cite{Eigen2.0}.
\end{modified}%
A port to MATLAB is also available~\cite{wagneriros11}. Similar mechanisms can be
applied in other (object-oriented) programming languages.

\subsection{Representing Manifolds in Software}\label{sec:software:mtk:manifold}
In \emph{MTK}, we represent \mplus-manifolds as C++ classes and require every
manifold to provide a common interface to be accessed by the generic
sensor fusion algorithm, \ie implementations of $\mplus$ and $\mmnus$.
The corresponding C++ interface is fairly straight-foward.  
\begin{modified} 
Defining a manifold requires an \lstinline"enum DOF" and two methods
\begin{cpp}[mathescape=true]
struct MyManifold {
	enum {DOF = $n$};
	typedef double scalar;
	void boxplus(vectview<const double, DOF> delta, double scale=1);
	void boxminus(vectview<double, DOF> delta, 
	         const MyManifold& other) const;
};
\end{cpp}
where $n$ is the degrees of freedom,
\lstinline"x.boxplus(delta, s)" implements $x:=x\mplus(s\cdot\delta)$
and \lstinline"y.boxminus(delta, x)" implements $\delta:=y\mmnus x$.
\lstinline"vectview" maps a \lstinline"double" array of size
\lstinline"DOF" to an expression directly usable in Eigen expressions.
The scaling factor in \lstinline"boxplus" can be set to $-1$ to
conveniently implement $x\mplus(-\delta)$.
\end{modified}

Additionally, a \mplus-manifold class can provide arbitrary member variables
and methods (which, \eg rotate vectors in the case of orientation) that are
specific to the particular manifold.

\emph{MTK} already comes with a library of readily available \emph{manifold
  primitive} implementations of $\R^n$ as \lstinline"vect<n>", \SO2
and \SO3 as \lstinline"SO2" and \lstinline"SO3" respectively, and
$S^2$ as \lstinline"S2".  It is possible to provide alternative
implementations of these or to add new implementations of other manifolds
basically by writing appropriate $\mplus$ and $\mmnus$ methods.

\subsection{Automatically Generated Compound Manifold Representations}\label{sec:software:mtk:build_manifold}

In practice, a single \emph{manifold primitive} is usually
insufficient to represent states (or measurements).  Thus, we also
need to cover \emph{compound \mplus-manifolds} consisting of several manifold
components in software.

A user-friendly approach would encapsulate the manifold in a class
with members for each individual submanifold which, mathematically,
corresponds to a Cartesian product.  Following this approach, \mpm
operators on the \emph{compound manifold} are needed, which use the
\mpm operators of the components as described in
Section~\ref{sec:compoundmanifolds}.  Using this method, the user
can access members of the manifold by name, and the algorithm just
sees the compound manifold.

This can be done by hand in principle, but becomes quite error-prone when there are
many components.
Therefore, \emph{MTK} provides a preprocessor macro which
generates a compound manifold from a list of simple manifolds.  The
way how \emph{MTK} does this automatically and hides all details from the
user is a main contribution of \emph{MTK}.

Returning to the INS example from the introduction where we need to
represent a state consisting of a position, an orientation, and a
velocity, we would have a nine-degrees-of-freedom state
\begin{equation}
\S = \R^3\times\SO3\times\R^3.
\end{equation}
Using our toolkit this can be constructed as:
\begin{cpp}
MTK_BUILD_MANIFOLD(state, 
	((vect<3>, pos))
	((SO3, orient))
	((vect<3>, vel))
)
\end{cpp}
Given this code snippet the preprocessor will generate a
\lstinline"class state", having public members \lstinline"pos",
\lstinline"orient", and \lstinline"vel".  Also generated are the
manifold operations \lstinline"boxplus" and \lstinline"boxminus" as well as the
total degrees of freedom \lstinline"state::DOF = 9" of the
\emph{compound manifold}.

The macro also addresses another technical problem: Kalman filters in
particular require covariance matrices to be specified which represent
process and measurement noise.  Similarly, individual parts of
covariance matrices often need to be analyzed.  In both cases the
indices of individual components in the flattened vector view need to
be known.  \emph{MTK} makes this possible by generating an enum
\lstinline"IDX" reflecting the index corresponding to the start of the
respective part of the vector view.  In the above example, the start
of the vectorized orientation can be determined using \eg
\lstinline"s.orient.IDX" for a \lstinline"state s".  The size of this
part is given by \lstinline"s.orient.DOF".
\begin{modified}
\emph{MTK} also provides convenience methods to access and modify corresponding
sub-vectors or sub-matrices using member pointers.
\begin{cpp}
// state covariance matrix:
Matrix<double, state::DOF, state::DOF> cov;
cov.setZero();

// set diagonal of covariance block of pos part:
setDiagonal(cov, &state::pos, 1);

// fill entire orient-pos-covariance block:
subblock(cov, &state::orient, &state::pos).fill(0.1);
\end{cpp}
\end{modified}

\subsection{Generic Least Squares Optimization and UKF Implementations}

Based on \emph{MTK}, we have developed a generic least squares
optimization framework called \emph{SLoM} (\emph{Sparse Least Squares
  on Manifolds})~\cite{Hertzberg2008} according to Table~\ref{tab:algo-LS} and \emph{UKFoM}~\cite{Wagner2010}, a generic UKF on manifolds
implementation according to Table~\ref{tab:algo-UKF}. Apart from handling
manifolds, \emph{SLoM} automatically infers sparsity, \ie which measurement depends 
on which variables, and exploits this in computing the Jacobian in \eqref{eq:ls:jacobian},
and in representing, and inverting it in \eqref{eq:ls:update}.

\emph{MTK}, \emph{SLoM} and \emph{UKFoM} are already available online on
\url{www.openslam.org/mtk} under an open source license.


\section{Experiments}\label{sec:experiments}

We now illustrate the advantages of the \bpm in terms of  both ease of
use and algorithmic performance.  

\subsection{Worked Example: INS-GPS Integration}

\begin{modified}
In a first experiment, we show how
\emph{MTK} and \emph{UKFoM} can be used to implement a minimalistic,
but working INS-GPS filter in about 50 lines of C++ code.  
The focus is on how the framework allows to concisely write
down such a filter. It integrates accelerometer and gyroscope
readings in the prediction step and fuses a global position
measurement (\eg loosely coupled GPS) in the measurement 
update. The first version assumes white noise. We then show
how the framework allows for an easy extension to colored noise.
\end{modified}


%
\begin{figure*}
  \centering
  \includegraphics{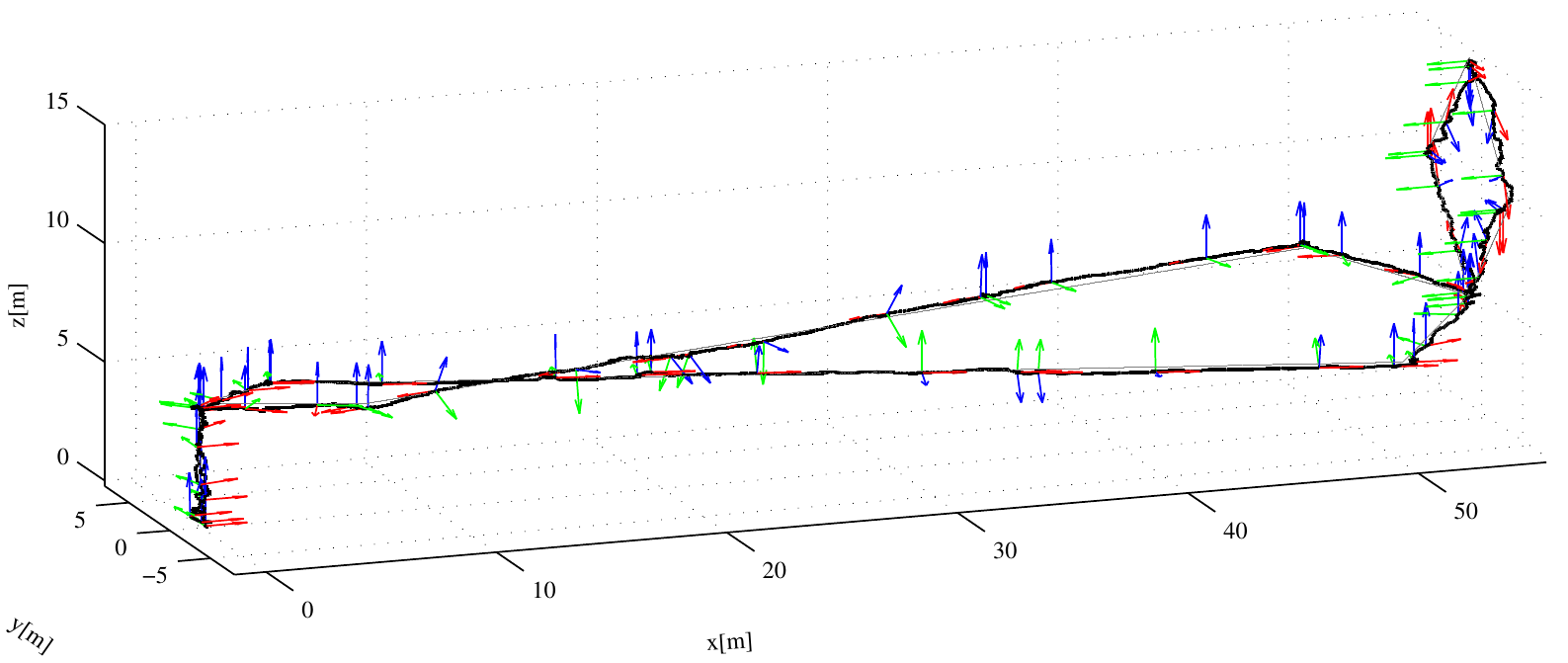}
  \caption{Trajectory estimated from a synthetic dataset by the
    \emph{UKFoM}-based minimal INS-GPS filter.  The trajectory starts
    and ends at $(0,0)$ and consists of a stylized figure eight with
    an embedded stylized loop (right).  While on the long segments of
    the figure eight the helicopter rotates about its roll axis.
}
    \label{fig:heli:ukf:3d}
\end{figure*}

\subsubsection{IMU Process Model}

Reusing the state manifold definition from Section~\ref{sec:software:mtk:build_manifold},
the process model implements $x_t = g(u_t, x_{t-1})$ (\cf
Section~\ref{sec:kalman:processandmeasurementmodels}) where the
control $u_t$ in this case comprises acceleration \lstinline"a" as
measured by a three-axis accelerometer and angular velocity
\lstinline"w" as measured by a three-axis gyroscope.
\begin{cpp}
state process_model(const state &s,
	const vect<3> &a, const vect<3> &w)
{
	state s2;
	
	// apply rotation
	vect<3> scaled_axis = w * dt;
	SO3 rot = SO3::exp(scaled_axis);
	s2.orient = s.orient * rot;

	// accelerate with gravity
	s2.vel = s.vel + (s.orient * a + gravity) * dt;

	// translate
	s2.pos = s.pos + s.vel * dt;

	return s2;
}
\end{cpp}

The body of the process model uses Euler integration of the motion
described by \lstinline"a" and \lstinline"w" over a short time
interval \lstinline"dt".
Note how individual components of the manifold state are accessed by
name and how they can provide non-trivial methods (overloaded
operators).

Further, we need to implement a function that returns the process
noise term $R_t$.
\begin{cpp}
ukf<state>::cov process_noise_cov()
{
	ukf<state>::cov cov = ukf<state>::cov::Zero();

	setDiagonal(cov, &state::pos, 0);
	setDiagonal(cov, &state::orient,gyro_noise*dt);
	setDiagonal(cov, &state::vel, acc_noise*dt);

	return cov;
}
\end{cpp}
\newcommand\sx{\sigma_p^2} \newcommand\sw{\sigma_\omega^2} \newcommand\sv{\sigma_v^2}
Note how the MTK's \lstinline"setDiagonal"
 function
automatically fills in the diagonal entries of the covariance matrix
such that their order matches the way the \mpm operators locally
vectorize the state space, \ie the user does not need to know about
these internals.  The constants 
\begin{modified}
$\sw=\;$\lstinline"gyro_noise" and $\sv=\;$\lstinline"acc_noise" are continuous noise
spectral densities for the gyroscope ($\tfrac{\text{rad}^2}{\text{s}}$) and accelerometer ($\tfrac{\text{m}^2}{\text{s}^3}$)
and multiplied by \lstinline"dt" in the process noise covariance matrix
\begin{equation} \newcommand\sd{\smash\vdots}
\newcommand\cd{&\cdots}
R_t =
    dt\cdot\operatorname{diag}(0,0,0,\sw,\sw,\sw,\sv,\sv,\sv).
\end{equation}
\end{modified}

\subsubsection{GPS Measurement Model}

Measurement models implement $z_t = h(x_t)$ (\cf
Section~\ref{sec:kalman:processandmeasurementmodels}),
in the case of a position measurement simply returning the 
position from $x_t$.
\begin{cpp}
vect<3> gps_measurement_model(const state &s)
{
	return s.pos;
}
\end{cpp}

We also need to implement a function that returns the measurement
noise term $Q_t$.
\begin{cpp}
Matrix3x3 gps_measurement_noise_cov()
{
	return gpos_noise * Matrix3x3::Identity();
}
\end{cpp}

Again, $\sigma_p^2=\;$\lstinline"gps_noise" is constant.  Note that
although we show a vector measurement in this example, \emph{UKFoM}
also supports manifold measurements.

\subsubsection{Executing the Filter}

Executing the UKF is now straight-forward.  We first instantiate the
\lstinline"ukf" template with our \lstinline"state" type and pass a default $(0,1,0)$
initial state to the constructor along with an initial
covariance matrix.  We then assume that sensor data is acquired in
some form of loop, and at each iteration execute the prediction and
correction (update) steps with the process and measurement models and
sensor readings as arguments.
\begin{cpp}
ukf<state>::cov init_cov = 
	0.001 * ukf<state>::cov::Identity();
ukf<state> kf(state(), init_cov);
vect<3> acc, gyro, gps;
for (...)
{
	kf.predict(
		boost::bind(process_model, _1, acc, gyro), 
		process_noise_cov);
		
	kf.update<3>(gps,
				 gps_measurement_model,
				 gps_measurement_noise_cov);
}
\end{cpp}

Note how our use of \lstinline"boost::bind" denotes an anonymous
function that maps the state \lstinline"_1" to
\lstinline"process_model(_1, acc, gyro)".

Also note that there is no particular order in which
\lstinline"predict()" and \lstinline"update()" need to be called, and
that there can be more than one measurement model -- typically one per
type of sensor data.

\begin{modified}

\subsubsection{Evaluation on a Synthetic Dataset}
\begin{figure}[t]
  \includegraphics{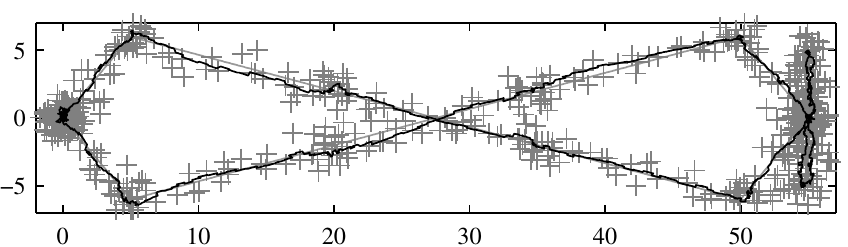}
  \includegraphics{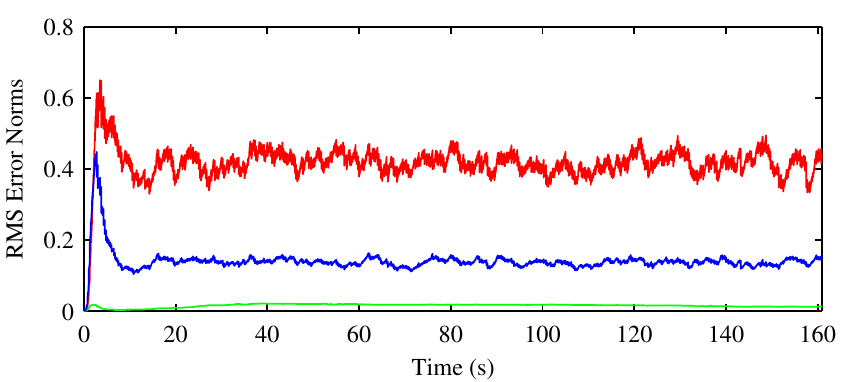}
  \caption{\textbf{Top}: Trajectory estimated by the minimal INS-GPS filter (black)
    vs. ground truth (solid gray) and position measurements (gray
    crosses) from a single filter run.  The plot shows the $x$-$y$-plane (m).
    \textbf{Bottom}: RMS error norms for position (red, m), orientation (green, rad)
    and velocity (blue, m/s) estimates from 50 Monte Carlo runs. The time averaged
    error is $0.415\,m$, $1.58\times 10^{-2}\,rad$ and $0.141\,m/s$, respectively.
    This magnitude seems plausible for a $\sigma_p=0.75\,\text{m}$ GPS. 
  }
    \label{fig:heli:ukf:xy}
\end{figure}
\begin{figure}
  \includegraphics{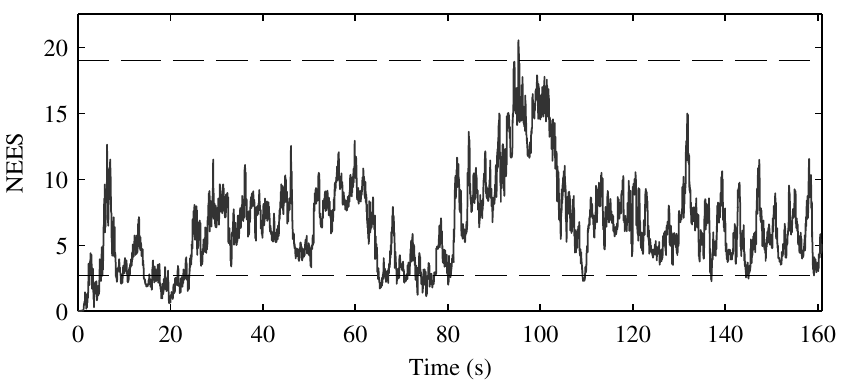}
  \includegraphics{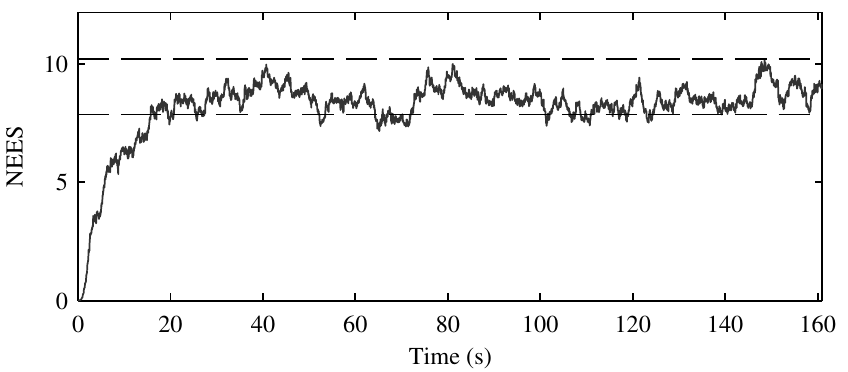}
  \includegraphics{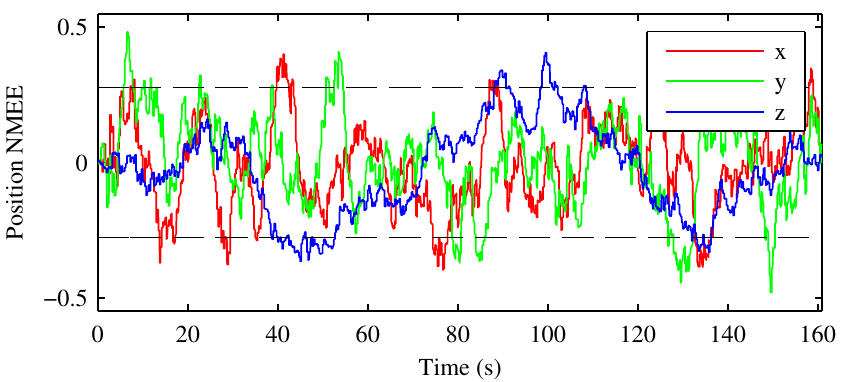}
  \includegraphics{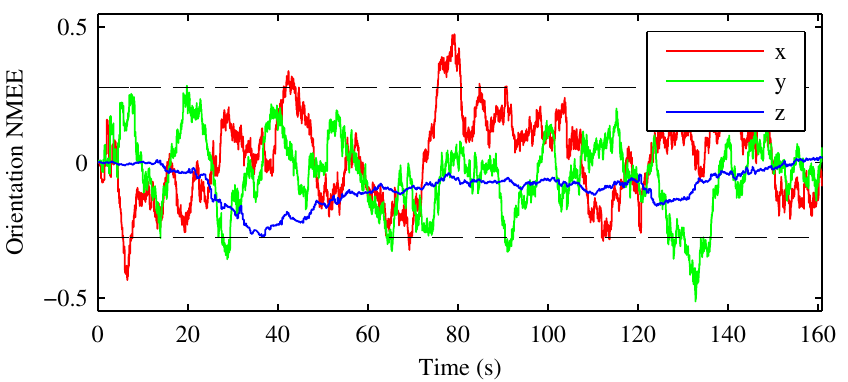}
  \includegraphics{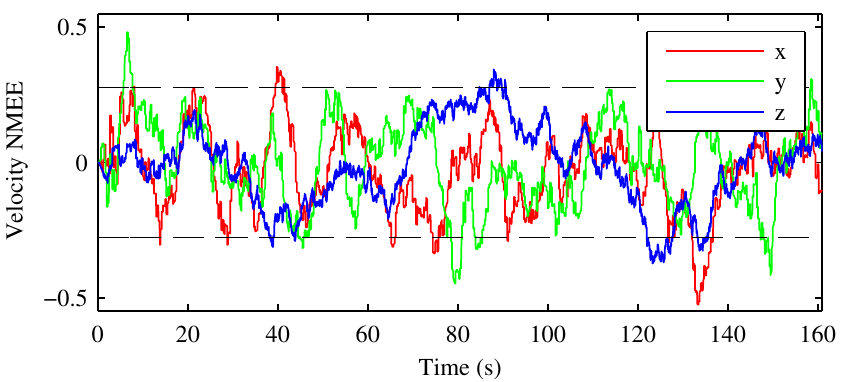}
  \caption{Plots for filter consistency evaluation (top to bottom):
    NEES $\norm{\mu_t\boxminus X_t}_{\Sigma_t}^2$
    from a single run; averaged NEES from 50 Monte Carlo runs; 
    $\frac{(\mu_t \mmnus X_t)_k}{\sqrt{\Sigma_{tkk}}}$
    averaged over 50 Monte Carlo runs (NMEE)
    with $k$ being x, y, and z of position,
    orientation and velocity . Note how
    each largely remains within its 95\% probability region (dashed),
    as should be the case for a consistent filter.}
    \label{fig:heli:ukf:nees-nmee}
\end{figure}

To conclude the worked example, we run the filter on a synthetic data
set consisting of sensor data generated from a predefined trajectory
with added white Gaussian noise.  The estimated 3D trajectory is shown
in Figure~\ref{fig:heli:ukf:3d}, its projection into the $x$-$y$-plane
compared to ground truth in Figure~\ref{fig:heli:ukf:xy}. 
Accelerometer and gyroscope readings are
available at 100\,Hz ($dt=0.01$) with white noise standard deviations
$\sigma_\omega=0.05\degree/\text{s}^{1/2}$ and $\sigma_v=2\,\text{mm}/\text{s}^{3/2}$ (MEMS class IMU).
\emph{GPS} is available at 4\,Hz, with white noise of $\sigma_p=0.75\,\text{m}$.

Solely based on
accelerometer and gyroscope data the state estimate would drift over
time. \emph{GPS} measurements allow the filter to reset errors
stemming from accumulated process noise.  However, over short time
periods accelerometer and gyroscope smooth out the noisy GPS measurements as
illustrated in Figure~\ref{fig:heli:ukf:xy}.  
As suggested by \cite{BarShalom2001} we verify the filter consistency
by computing the \emph{normalized estimation error squared} (NEES) 
and the \emph{normalized mean estimation error} (NMEE) for each
state component~(Fig.\@~\ref{fig:heli:ukf:nees-nmee}).
Figure \ref{fig:heli:ukf-euler:rms-error-norms} and \ref{fig:heli:ukf-scaledaxis:rms-error-norms} present the same results
for a UKF using Euler angles and scaled axis respectively, showing the expected failure once the orientation
approaches singularity. 
Figure \ref{fig:heli:ukf-quaternion:rms-error-norms} shows the results for the technique
proposed in \cite{Merwe2004SPKF} with a plain quaternion $\R^4$ in the UKF state and 
a process model that drives the quaternion towards normalization ($\eta=0.1/\text{s}$, cf. Sec.\@~\ref{sec:ukfrelated}).
Overall, Euler-angle and scaled axis fail at singularities, the \mplus-method is very slightly better than the plain quaternion.
The difference in performance is not very relevant, our claim is rather that the \mplus-method is conceptually more elegant.
Computation time was 21/32\,$\mu$s (\mplus), 28/23\,$\mu$s (Euler), 25/23\,$\mu$s
(scaled axis), 21/33\,$\mu$s (quaternion) for a predict-step/for a GPS update.
All timings were determined on an Intel Xeon CPU E5420 @2.50GHz
running 32bit Linux.
\begin{figure}[t]
   \includegraphics{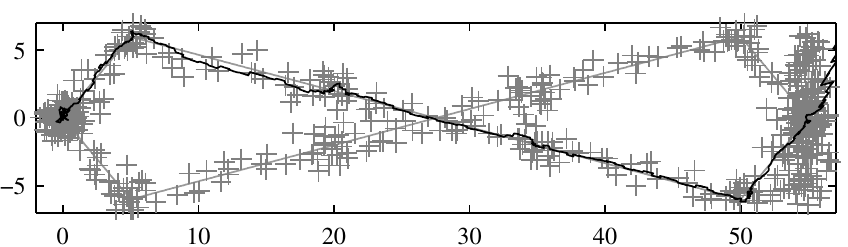}
   \includegraphics{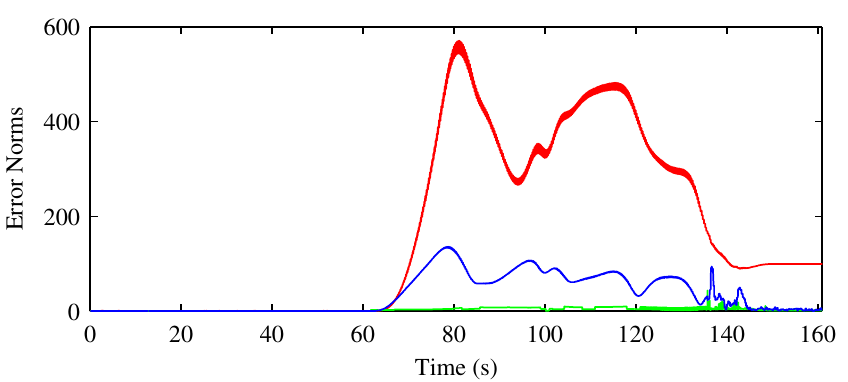}
   \includegraphics{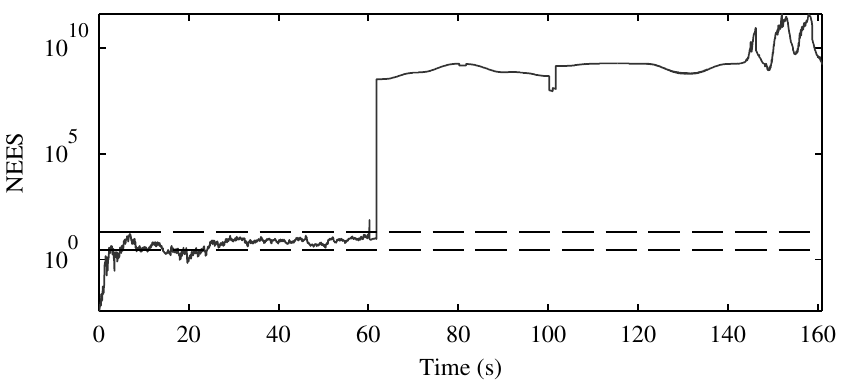}
   \caption{Performance of an UKF using Euler angles for the orientation (top to bottom):
            Trajectory estimated, error norms as in Fig.\@~\ref{fig:heli:ukf:xy}, 
            NEES, from a single run. The filter operates normally until $t=60\,\text{s}$ where the orientation
            approaches the singularity, the filter becomes inconsistent and the error rapidly increases with the estimate
            leaving the workspace. }
   \label{fig:heli:ukf-euler:rms-error-norms}
\end{figure}
\begin{figure}[t]
   \includegraphics{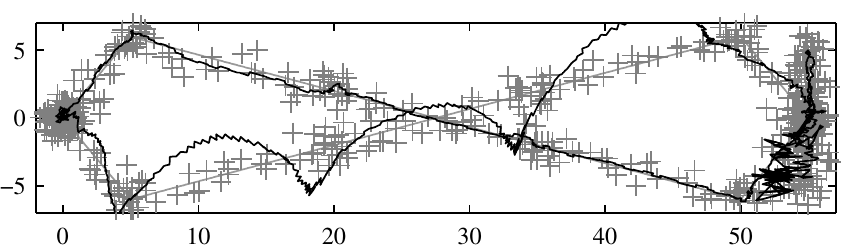}
   \includegraphics{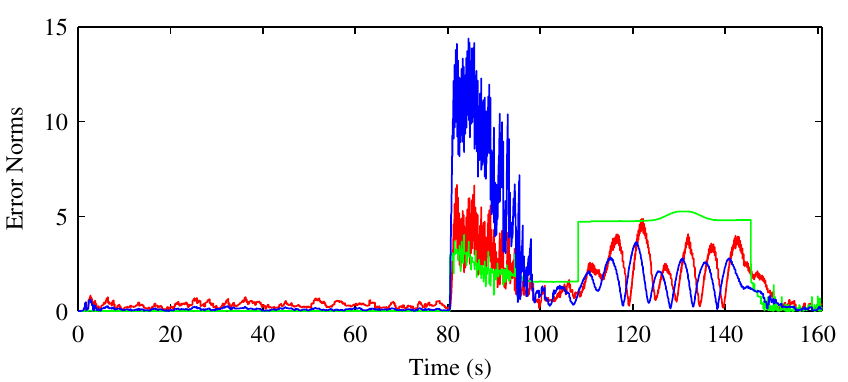}
   \includegraphics{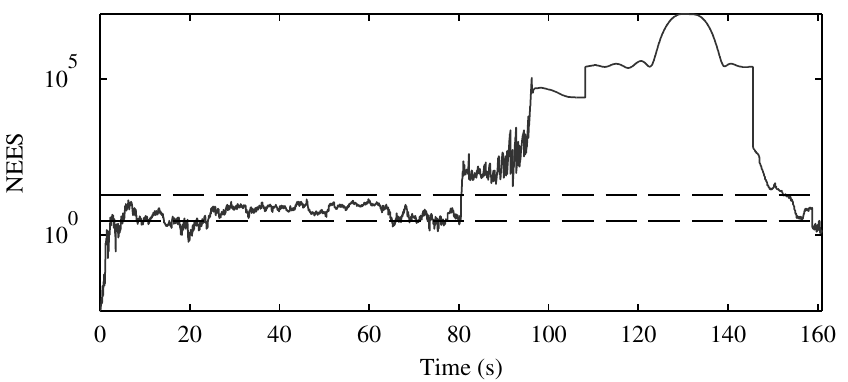}
   \caption{Performance of an UKF using a scaled axis representation for the orientation (top to bottom):
            Trajectory estimated, error norms as in Fig.\@~\ref{fig:heli:ukf:xy}, 
            NEES, from a single run. The filter operates normally until $t=80\,\text{s}$ where the orientation
            approaches the singularity, the filter becomes inconsistent and the error rapidly increases.
            Surprisingly, the filter recovers in the end, showing that scaled axis is a more robust
            representation than Euler angles.}
   \label{fig:heli:ukf-scaledaxis:rms-error-norms}
\end{figure}
\begin{figure}[t]
  \includegraphics{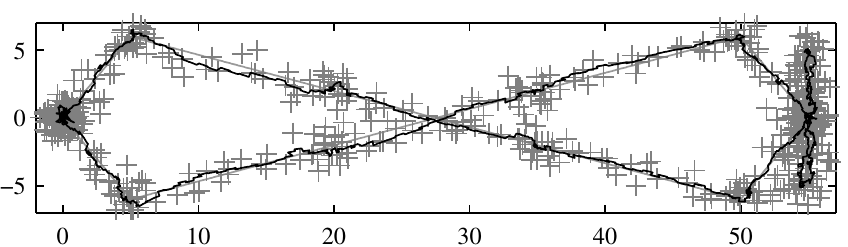}
  \includegraphics{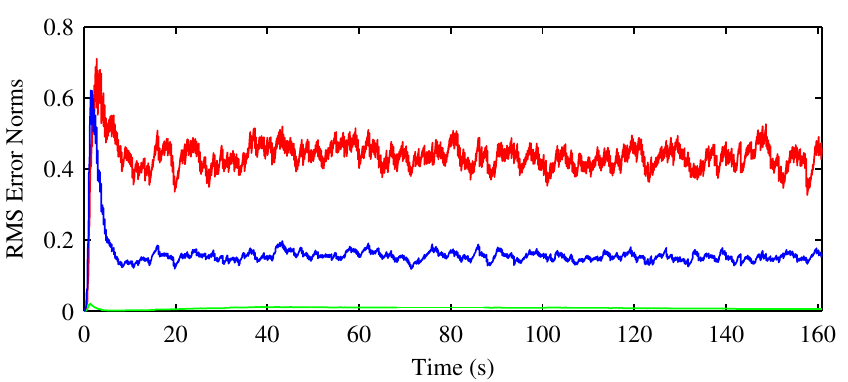}
  \includegraphics{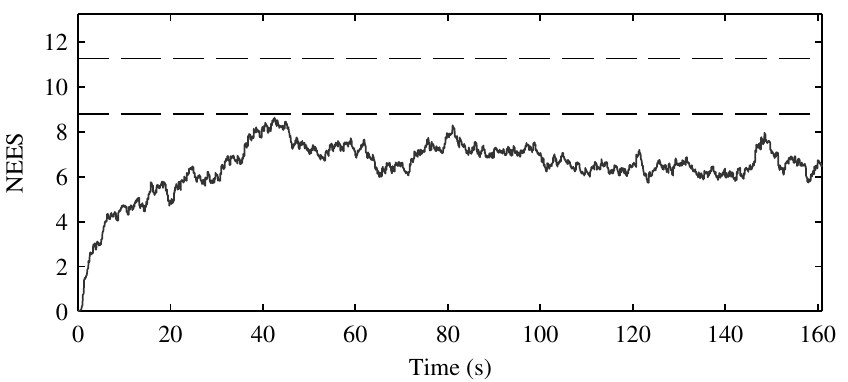}
   \caption{Performance of an UKF using a plain quaternion (top to bottom):
            Trajectory estimated, RMS error norms as in Fig.\@~\ref{fig:heli:ukf:xy}, 
            NEES, from 50 Monte Carlo runs. 
            The NEES is slightly too low (ca. by 1), probably because by using $\sigma_\omega^2$ 
            as process noise covariance in all 4 quaternion components, the filter ``thinks''
            there is process noise on the norm of the quaternion, while in fact there
            is none. The time-averaged error is $0.434\,m$, $1.71 \times 10^{-2}\,rad$, and $0.160\,m/s$ in position, orientation,
            and velocity, respectively. This is slightly worse than for the \mplus-method, probably caused
            by the fact, that the quaternion is not fully normalized making the filter use
            the false DOF created by the quaternion's norm to fit to the measurements.
            }
   \label{fig:heli:ukf-quaternion:rms-error-norms}
\end{figure}
\end{modified}

\begin{modified}
\subsection{Extension to Colored Noise Errors}

GPS errors are correlated and hence a INS-GPS filter
should model colored not white noise. Therefor, a
bias vector must be added to the state:
\begin{cpp}
MTK_BUILD_MANIFOLD(state, 
	...
	((vect<3>, gps_bias))
)
\end{cpp}
The bias $b=$\lstinline"gps_bias" follows the process model
\begin{align*}
    b_{t+1} = \exp\left(-\tfrac{dt}{T}\right) b_t + \Ndp{0}{\left(1\!-\!\exp\left(-\tfrac{2dt}{T}\right)\right)\sigma_{b}^2\I_3}
\end{align*}
which realizes an autocorrelation with a given variance
$\Cov(b_t)=\sigma_{b}^2\I_3$ and specified exponential decay $\cor(b_t,b_{t+k})=\exp\left(-\tfrac{k \cdot dt}{T}\right)$.
The formula is taken from the textbook by Grewal~\cite[(8.76), (8.78)]{Grewal01}
and implemented in \lstinline"process_model" by
\begin{cpp}
	s2.gps_bias = exp(-dt/T_pos) * s.gps_bias;
\end{cpp}
and in \lstinline"process_noise_cov" by
\begin{cpp}
	setDiagonal(cov, &state::gps_bias, 
		gps_cnoise*(1-exp(-2*dt/T_pos)));
\end{cpp}
The initial covariance is also set to $\sigma_{b}^2$:
\begin{cpp}
	setDiagonal(init_cov, &state::gps_bias, gps_cnoise);
\end{cpp}
Finally, \lstinline"gps_measurement_model" adds the bias:
\begin{cpp}
	return s.pos+s.bias;
\end{cpp}
Figure \ref{fig:heli:ukf-plots-colored:rms-error-norms} shows the
performance of the modified filter with a simulation that includes
colored noise on the GPS measurement ($\sigma_b^2=5\text{m}$, $T=1800\text{s}$).

This example shows that \emph{MTK} and \emph{UKFoM} allow for
rapidly trying out different representations and models without
being hindered by implementing bookkeeping issues.
In a similar way, omitted here for lack of space, gyroscope and accelerometer
bias can be integrated. Beyond that, further improvement would
require operating on GPS pseudo-ranges (tightly coupled setup),
with the state augmented by biases compensating for the various
error sources (receiver clock error, per-satellite clock errors,
ephemeris errors, atmospheric delays, etc.; \cite[Ch. 5]{Grewal01}).

\begin{figure}[t]
   \includegraphics{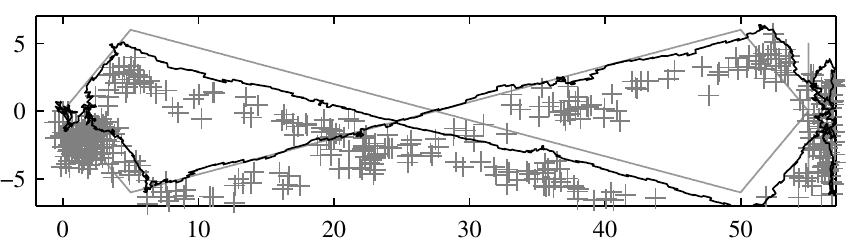}
   \includegraphics{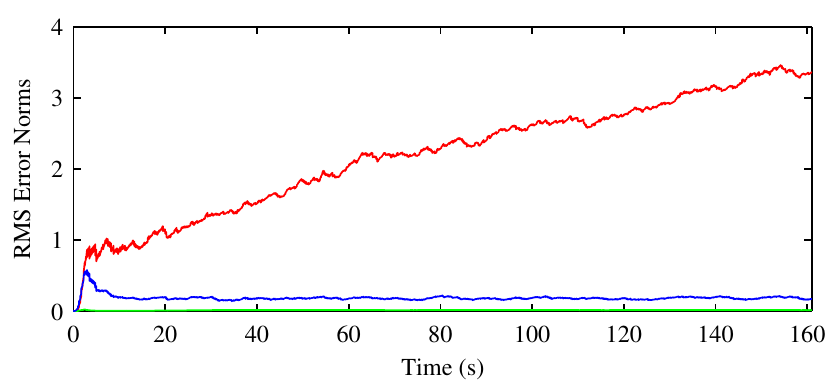}
   \includegraphics{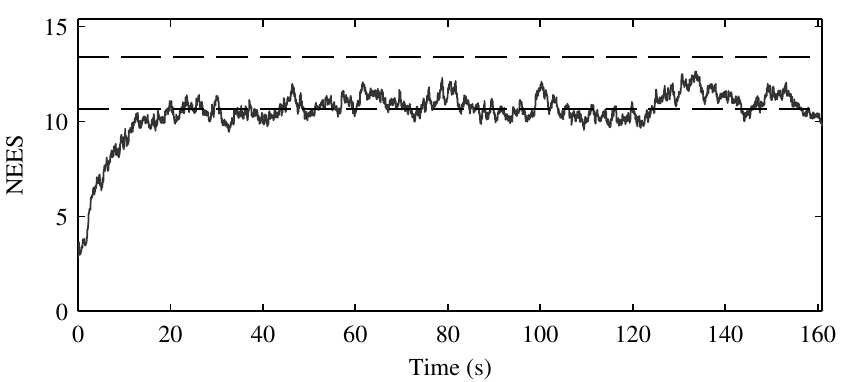}
   \caption{Performance of the \mplus-method UKF with colored noise (top to bottom):
            Trajectory estimated, error norms as in Fig.\@~\ref{fig:heli:ukf:xy}, 
            NEES from 50 Monte Carlo runs. The filter is consistent, notably the position
            error grows over time. This is as expected: The filter knows its initial
            position and from this it can initially deduce the GPS-bias
            with about $1\text{m}$ precision. However, over time
            the bias drifts and
            with the inertial system being to imprecise, their is no information
            on the new bias and hence the position error grows. Velocity and
            orientation error keep low, because these are deduced from the relative
            position of GPS measurements where the bias cancels out.
            }
   \label{fig:heli:ukf-plots-colored:rms-error-norms}
\end{figure}

\end{modified}

\subsection{Pose Relation Graph Optimization}
\begin{modified} 
To show the benefit of our \mplus-approach we optimized several 3D pose
graphs using our manifold representation and compared it to the singular
representations of Euler angles and matrix exponential (see Section~\ref{sec:bpm:3dMatrix}), as well
as a four dimensional quaternion representation. When using Gauss-Newton
optimization the latter would fail, due to the rank-deficity of the problem, so
we added the pseudo measurement $|q|=1$.

First, we show that the \mplus-method works on real-world data sets. 
Figure~\ref{fig:dlr} shows a simple 2D landmark SLAM problem
(DLR/Spatial Cognition data set~\cite{SpatialCognitionDataSet}). 
Figure ~\ref{fig:stanfordRSS} shows
the 3D Stanford multi-storey parking garage data set, where the initial estimate is so good,
most methods work well. 
\begin{figure}
\centering
  \setbox0\vbox{\vskip .5cm\hbox to 0em{%
  \includegraphics[width=\columnwidth]{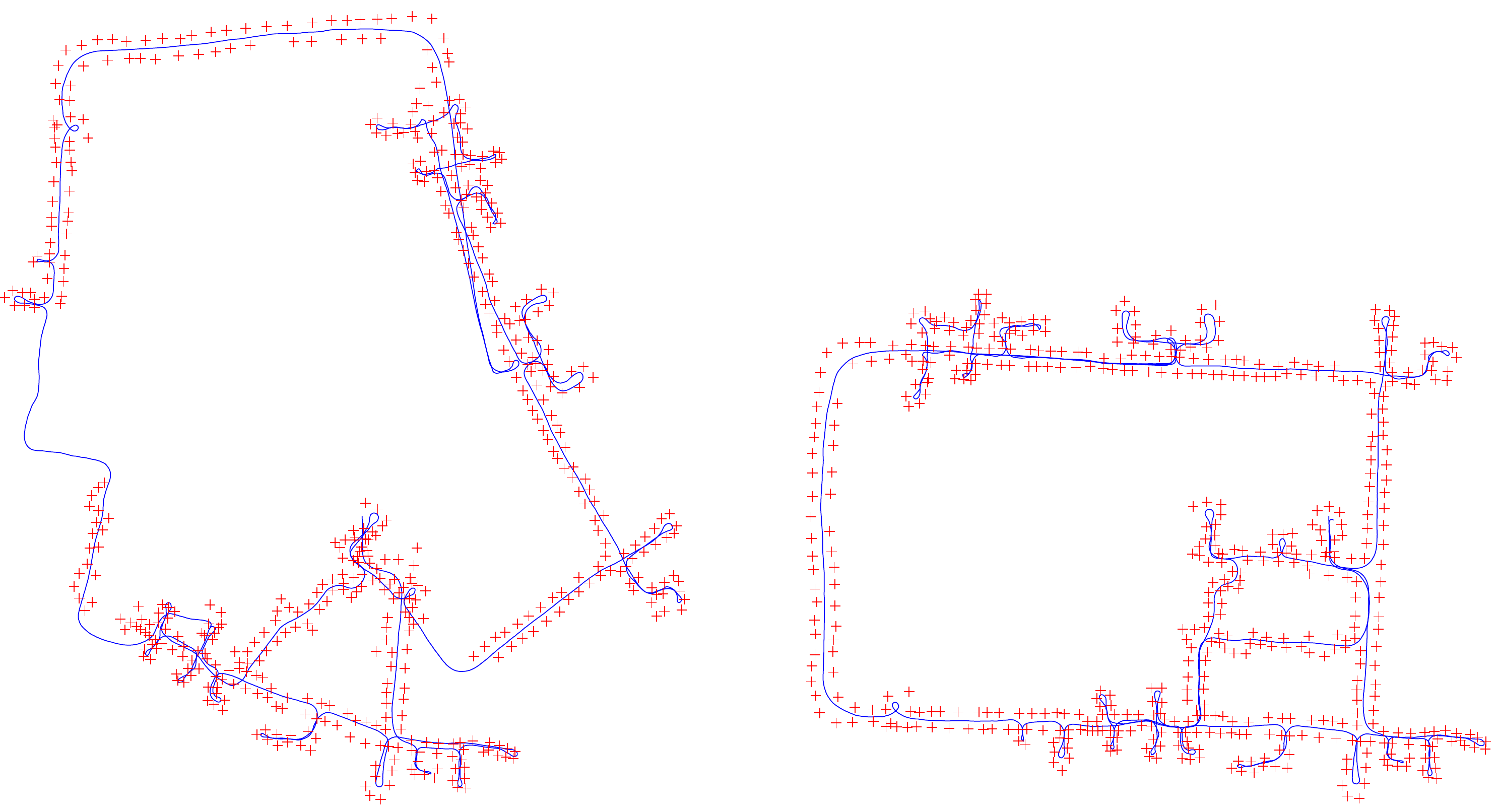}\hss}}%
  \hbox{\copy0%
  \vbox to 0ex{\vskip-\ht0{\vbox{\hfill\includegraphics{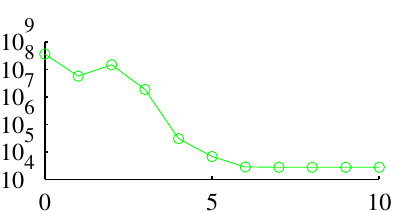}}}\vss}
  }
  \caption{The DLR data set \cite{SpatialCognitionDataSet} before (left) and after (right)
  Gauss-Newton optimization. We also show the residual sum of squares over
  iteration steps (top, right). No comparison is made, as in 2D
the \mplus-operator only encapsulates the handling of angular periodicity.}
  \label{fig:dlr}
\end{figure}

\begin{figure}
\centering
  \includegraphics[width=\columnwidth]{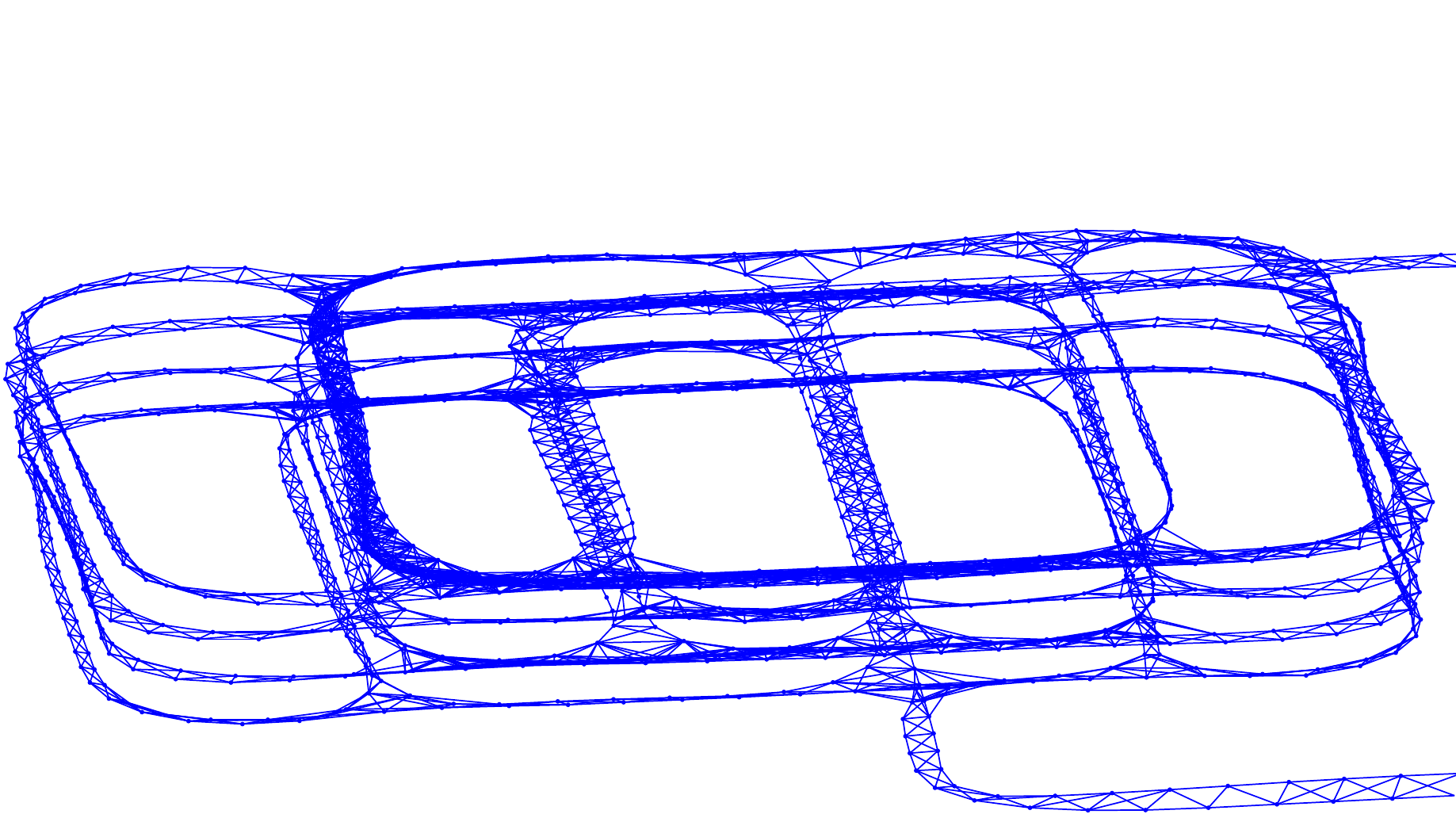}
  \includegraphics{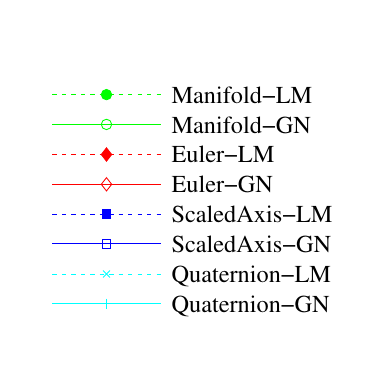} \;
  \includegraphics{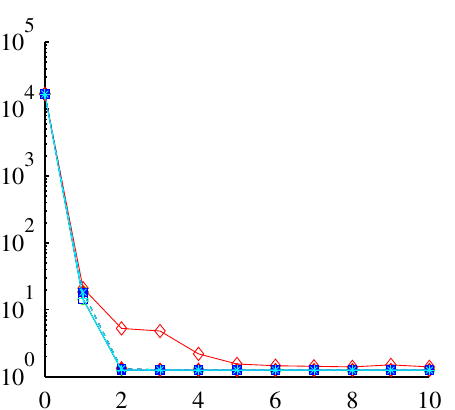}
  
  \caption{Residual sum of squares over iteration steps of Gauss-Newton (GN) and
  Levenberg-Marquardt (LM) optimization on the Stanford parking garage data set as used in~\cite{Grisetti2009Nonlinear}.
  Gauss-Newton with Euler-angles is clearly inferior to all
  other representations, however being far away
  from singularities, it still converges.}
  \label{fig:stanfordRSS}
\end{figure}

Second, for a quantitative comparison, we use the simulated dataset from~\cite[supplement]{KaessTRO2008}, 
shown in Figure~\ref{fig:iSAM:result:manifold} and investigate
how the different state representations behave under increasing noise levels.
Figure~\ref{fig:sphere400RSS} shows the results.

\begin{figure}[tbp]
  \includegraphics{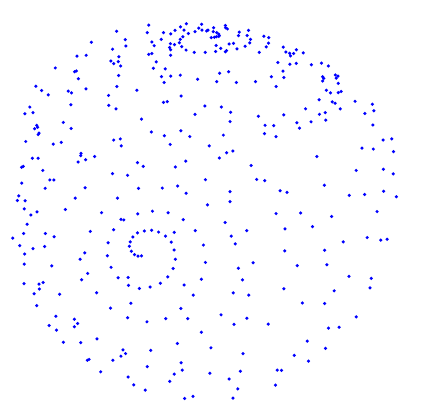}
  \includegraphics{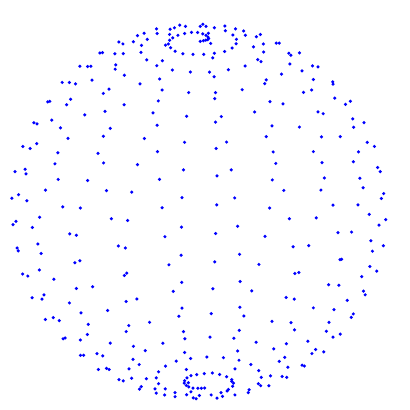}
  \caption{The sphere400 dataset~\cite{KaessTRO2008} generated by a
    virtual robot driving on a 3D sphere. It consists of a set
    of 400~three-dimensional poses and about 780~noisy constraints
    between them. The constraints stem from motion between consecutive poses or
    relations to previously visited poses. Poses are
    initialized from motion constraints (left) and then optimized with our \emph{SLoM} framework.
    }
    \label{fig:iSAM:result:manifold}
\end{figure}

\begin{figure}
\centering
  \includegraphics{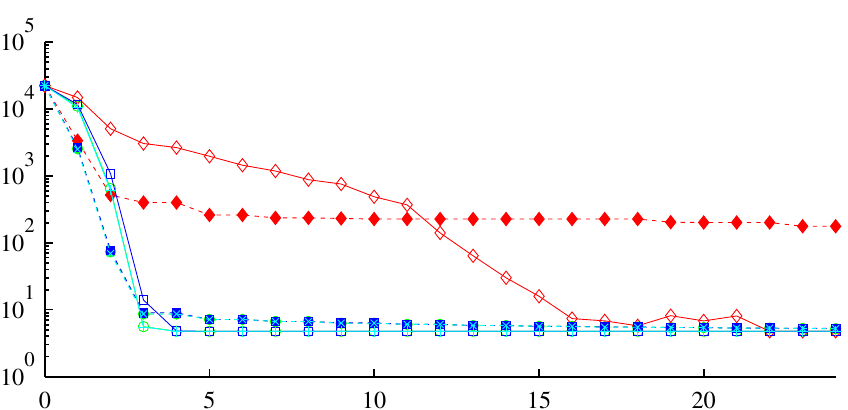}
  \includegraphics{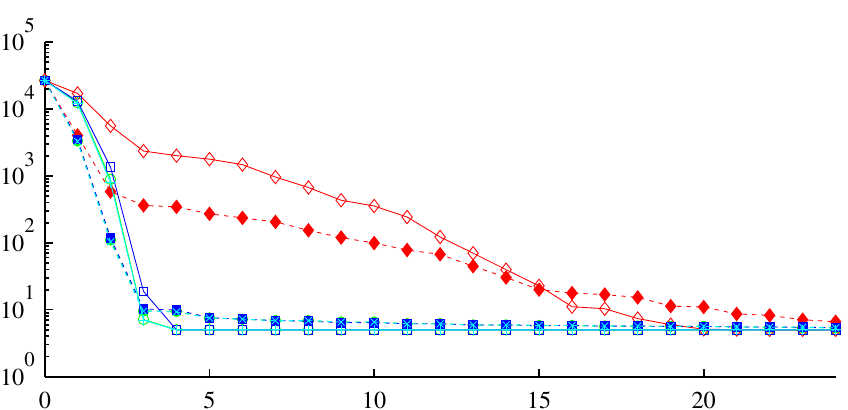}
  \includegraphics{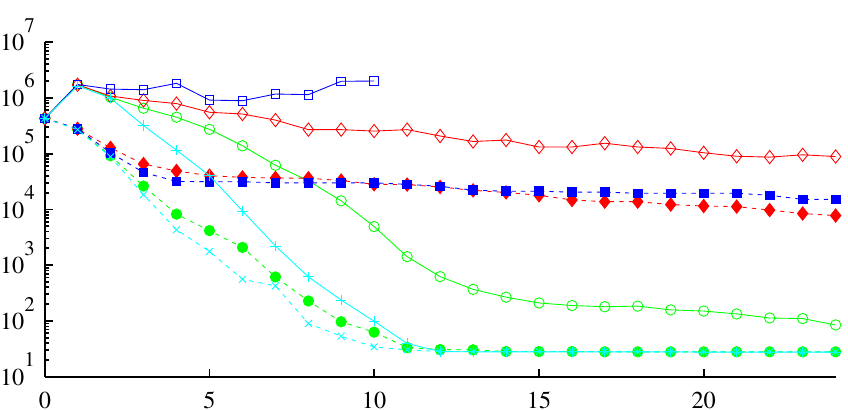}
  
  \caption{Residual sum of squares over iteration steps of Gauss-Newton and
  Levenberg-Marquardt with different state representations on the dataset
  in Fig.\@~\ref{fig:iSAM:result:manifold} (see Figure~\ref{fig:stanfordRSS} for legend).
  The same data set was optimized using original noise (top) and with additional
  noises of 0.01 (middle) and 0.1 rad/m (bottom). For the latter two, 
  the median of 31 runs is plotted. Plots ending unfinished indicate that more
  than half the optimizations could not be finished due to a singularity.
  The \mplus-method clearly out-performs the singular
  representations. For the high noise-level the
  quaternion representation is slightly better. However, when comparing run-times instead of
  iteration counts, the latter is slower due to the additional measurements
  and the larger state-space (7~DOF instead of 6~DOF per pose). 
  Computation times were 63\,ms for the 4D quaternion
  and 42\,ms for the \mplus-approach per step. This fits well to 
  the nominal factor of $(\tfrac76)^3\approx1.59$ for the $O(n^3)$-matrix decomposition.
  For  Euler angle and matrix exponential,
  the evaluation took longer with
  times of 80\,ms and 62\,ms per step probably due to not hand-tuned code and 
  the high number of trigonometric functions involved. 
   }
  \label{fig:sphere400RSS}
\end{figure}

\end{modified}


\section{Conclusions}

We have presented a principled way of providing a local vector-space
view of a manifold $\S$ for use with sensor fusion algorithms. We have achieved this by means of an operator
$\mplus:\S\times\R^n\to\S$ that adds a small vector-valued
perturbation to a state in $\S$ and an inverse operator
$\mmnus:\S\times\S\to\R^n$ that computes the
vector-valued perturbation turning one state into another. A space equipped
with such operators is called a \mplus-manifold.

We have axiomatized this approach and lifted the concepts 
of Gaussian distribution, mean, and covariance to \mplus-manifolds
therewith. The \mpm operators allow for the integration of manifolds into generic estimation
algorithms such as least-squares or the UKF mainly by replacing $+$ and $-$ with
$\mplus$ and $\mmnus$. For the UKF additionally the 
computation of the mean and the covariance update are modified.

The \bpm is not only an abstract mathematical framework but also a
software engineering toolkit for implementing estimation
algorithms. In the form of our Manifold Toolkit (\emph{MTK})
implementation (and its MATLAB variant MTKM), it automatically derives \mpm operators for compound
manifolds and mediates between a flat-vector and a
structured-components view.


\hyphenation{SFB/-TR}
\section{Acknowledgements}

This work has been partially supported by the German Research
Foundation (DFG) under grant SFB/TR~8 Spatial Cognition, as well as by
the German Federal Ministry of Education and Research (BMBF) under
grant 01IS09044B and grant 01IW10002 (SHIP -- Semantic Integration of Heterogeneous Processes).



\appendix

\section*{Appendix}

\section{Mathematical Analysis of \mplus-Manifolds}\label{sec:formalization}

\begin{modified}

In Section \ref{sec:BoxplusMethod} we have introduced the $\mplus$-method
from a conceptual point of view, including the axiomatization of $\mplus$-manifolds and the generalization of the 
probabilistic notions of expected value, covariance, and Gaussian distribution
from vector spaces to $\mplus$-manifolds. We will now underpin this discussion
with mathematical proofs.

\subsection{\mplus-Manifolds}

First, we recall the textbook definition of manifolds~(see,
e.g.,~\cite{SmoothManifolds}). We simplify certain aspects not
relevant within the context of our method; this does not affect formal
correctness. In particular, we view a manifold $\S$ as embedded as a
subset $\S\subset\R^s$ into Euclidean space $\R^s$ from the outset;
this is without loss of generality by Whitney's embedding theorem, and
simplifies the presentation for our purposes.

\hyphenation{homeo-mor-phism}
\begin{dfn}[Manifold~\cite{SmoothManifolds}]\label{dfn:manifold}
  \newcommand\aInA{{\alpha\in A}} A ($C^k$-, or smooth) \defi{manifold} is a pair
  $(\S,\F)$ (usually denoted  just $\S$) consisting of a connected set
  $\S\subset\R^s$ and an \emph{atlas} $\F=(U_\alpha,
  \phi_\alpha)_\aInA$, i.e.\ a family of \emph{charts}
  $(U_\alpha,\phi_\alpha)$ consisting of an open subset $U_\alpha$ of
  $\S$ and a homeomorphism $\phi_\alpha:U_\alpha\to V_\alpha$ of
  $U_\alpha$ to an open subset $V_\alpha\subset\R^n$. Here, $U_\alpha$
  being open in $\S$ means that there is an open set
  $\tilde{U}_\alpha\subset\R^s$ such that $U_\alpha=\S\cap\tilde{U}_\alpha$. These data are
  subject to the following requirements.
\begin{enumerate}
  \item The charts in $\F$ cover $\S$, i.e.\ $\S=\bigcup_\aInA U_\alpha$.
  \item \label{dfn:manifold:smoothtransition}
        If $U_\alpha\cap U_\beta\neq\emptyset$, the \defi{transition map}
\begin{equation} \label{eq:manifold:smoothtransition}
\phi_\alpha\circ\phi_\beta\inv:
     \phi_\beta(U_\alpha\cap U_\beta)
     \to\phi_\alpha(U_\alpha\cap U_\beta)
\end{equation}
is a $C^k$-diffeomorphism.
\end{enumerate}
The number $n$ is called the \defi{dimension} or the number of
\defi{degrees of freedom} of $\S$.
\end{dfn}

We recall the generalization of the definition of \emph{smoothness},
i.e.\@ being $k$ times differentiable, to functions defined on
arbitrary (not necessarily open) subsets $\S\subset\R^s$:

\begin{dfn}[Smooth Function]
\label{dfn:smooth-function}
For $\S\subset\R^s$, a function $f:\S\to\R^n$ is called \defi{smooth},
i.e. $C^k$, in $x\in\S$ if there exists an open neighbourhood
$U\subset\R^s$ of $x$ and a smooth function $\tilde f:U\to\R^n$ that
extends $f|_{U\cap\S}$.
\end{dfn}

Next we show that every \mplus-manifold is indeed a manifold, justifying the name. 
The reverse is not true in general.

\begin{lem}\label{lem:boxplus=manifold}
Every \mplus-manifold is a manifold, with the atlas $(U_x,\phi_x)_{x\in\S}$ where
\begin{gather}
  U_x=\setwith{y}{y\mmnus x\in V}\\
   \phi_x: U_x \to V,\; y \mapsto y\mmnus x.
\end{gather}
\end{lem}
(This result will be sharpened later in
Corollary~\ref{corl:embedded-submanifold}.)
\begin{proof}
  $\S$ is connected, as $\gamma:\lambda \mapsto x\mplus(\lambda (y\mmnus x))$
  is a path from $\gamma(0)=x$ by \eqref{ax:zero} to $\gamma(1)=y$ by \eqref{ax:surj}.
  From \eqref{ax:inj} we have that $\phi_x\inv(\delta)=x\mplus\delta$
  is injective on $V$, and therefore bijective onto its image
  $U_x$. As both $\phi_x$ and $\phi_x\inv$ are required to be smooth,
  $\phi_x$ is a diffeomorphism, in particular a homeomorphism. The set
  $U_x$ is open in $\S$, as it is the preimage of $V$ under the
  continuous function $\phi_x$, and since $x\in U_x$ we have
  $\S=\bigcup_{x\in\S}U_x$.
  Finally, the transition map $\phi_x\circ\phi_y\inv$ is a
  composite of diffeomorphisms and therefore diffeomorphic.
\end{proof}

\subsection{Induced Metric}

\begin{lem}\label{lem:metric}
  The operation \mmnus defines a metric $\d$ on $\S$ by
\begin{equation}\label{eq:form:metric-dfn}
\d_\S(x,y):=\norm{y\mmnus_\S x}.
\end{equation}
\end{lem}
\begin{proof}
Positive definiteness of $\d$ follows from  Axiom~\eqref{ax:zero} and
positive definiteness of $\norm{\cdot}$.

Symmetry can be shown using \eqref{ax:triangle}:
\begin{align}
\d(x,y)&=
\norm{y\mmnus x}
  =\norm{(y\mplus0)\mmnus(y\mplus(x\mmnus y))} \nonumber\\
  &\le\norm{x\mmnus y} = \d(y,x)
\end{align}
and symmetrically, which implies equality.
The triangle inequality also follows from \eqref{ax:triangle}:
\begin{align}
\d(x,z)
  &= \norm{z\mmnus x}\\
  &= \norm{(y\mplus(z\mmnus y))\mmnus(y\mplus(x\mmnus y))}\\
  &\le\norm{(z\mmnus y)-(x\mmnus y)}\\
  &\le\norm{(x\mmnus y)}+\norm{(z\mmnus y)}\\
  &= \d(x,y) + \d(y,z) \qedhere
\end{align} 
\end{proof}

\subsection{Smooth Functions} 
\label{sec:form:smooth-func}

\begin{lem}\label{lem:smooth-equiv}
  For a map $f:\S\to\M$ between $\mplus$-manifolds $\S$ and $\M$, the
  following are equivalent for every $x\in\S$:
\begin{enumerate}
\item $f$ is smooth in $x$ (Definition~\ref{dfn:smooth-function})
        \label{lem:smooth-equiv1}
      \item $f(x\mplus_\S\delta)\mmnus_\M z$ is smooth in $\delta$ at
        $\delta=0$ whenever $z\in\M$ is such that $f(x)\in U_z$.
        \label{lem:smooth-equiv2}
\end{enumerate}
\end{lem}
\begin{proof}
  \ref{lem:smooth-equiv1} implies \ref{lem:smooth-equiv2}, as the
  concatenation of smooth functions is smooth.

  For the converse implication, fix $z$ as in \ref{lem:smooth-equiv2},
  let $\S\subset\R^s$ and $\M\subset\R^m$, and let $U\subset\R^s$ be a
  neighbourhood of $x$ such that \mmnus extends smoothly to
  $U\times\{x\}$. Then we extend $f$ smoothly to $\tilde f:U\to\R^m$ by
  \begin{align} 
    \tilde f(y)&= z \mplus_\M (f (x \mplus_\S (y \mmnus_\S x)) \mmnus_\M z).
    \qedhere
  \end{align}
\end{proof}

Replacing \mpm in $\delta\mapsto f(x\mplus\delta)\mmnus z$ by the
induced charts as in Lemma~\ref{lem:boxplus=manifold}, we see that
smoothness corresponds to the classical definition of smooth functions
on manifolds \cite[p.\,32]{SmoothManifolds}:
\begin{equation}
f(x\mplus\delta)\mmnus z = \phi_z(f(\phi_x\inv(\delta))),
\end{equation}
with the right hand side required to be smooth in $\delta$ at
$\delta=\phi_x(x)=0$.

A direct consequence of this fact is
\begin{corl}\label{corl:embedded-submanifold}
  Every $\mplus$-manifold $\S\subset \R^s$ is an embedded submanifold of $\R^s$.
\end{corl}
(Recall that this means that the embedding $\S\into\R^s$ is an
immersion, i.e.\ a smooth map of manifolds that induces an injection
of tangent spaces, and moreover that the topology of the manifold $\S$
is the subspace topology in $\R^s$~\cite{SmoothManifolds}.)
\begin{proof}
  $\S$ carries the subspace topology by construction.  Clearly, the
  injection $\S\into\R^s$ is smooth according to
  Definition~\ref{dfn:smooth-function}, and hence as a map of
  manifolds by the above argument. Since tangent spaces are spaces of
  differential operators on smooth real-valued functions and moreover
  have a local nature~\cite{SmoothManifolds}, the immersion property
  amounts to every smooth real-valued function on an open subset of
  $\S$ extending to a smooth function on an open subset of $\R^s$,
  which is precisely the content of
  Definition~\ref{dfn:smooth-function}.
\end{proof}

\subsection{Isomorphic \mplus-Manifolds}
For every type of structure, one has a notion of homomorphism, which
describes mappings that preserve the relevant structure. The natural
notion of morphism $\phi:\S\to\M$ of \mplus-manifolds is a smooth map
$\phi:\S\to\M$ that is homomorphic w.r.t.\ the algebraic operations
$\mplus$ and $\mmnus$, i.e.\
\begin{align}
    \phi(x\mplus_\S\delta) & = \phi(x) \mplus_\M \delta\\
    \phi(x\mmnus_\S y) & = \phi(x) \mmnus_\M \phi(y).
\end{align}
As usual, an isomorphism is a bijective homomorphism whose inverse is
again a homomorphism. Compatibility of inverses of homomorphisms with
algebraic operations as above is automatic, so that an isomorphism
$\phi:\S\to\M$ of $\mplus$-manifolds is just a diffeomorphism
$\phi:\S\to\M$ (i.e.\ an invertible smooth map with smooth inverse)
that is homomorphic w.r.t.\ $\mplus$ and $\mmnus$.
If such a $\phi$ exists, $\S$ and $\M$ are \emph{isomorphic}. It is
clear that isomorphic \mplus-manifolds are indistinguishable as such,
i.e.\ differ only w.r.t.\ the representation of their elements.  We
will give examples of isomorphic $\mplus$-manifolds in Appendix
\ref{sec:examplemanifolds}; e.g.\@ orthonormal matrices and
unit quaternions form isomorphic $\mplus$-manifolds.

\subsection{Defining Symmetric \mplus-Manifolds}\label{sec:form:symmetric-encaps}
Most \mplus-manifolds arising in practice are manifolds with inherent
symmetries. This can be exploited by defining \mplus at one reference element
and pulling this structure back to the other elements along the
symmetry.

Formally, we describe the following procedure for turning an
$n$-dimensional manifold $\S$ with sufficient symmetry into a
\mplus-manifold. The first step is to define a smooth and surjective
function $\psi:\R^n\to\S$ which is required to be locally
diffeomorphic, \ie for a neighborhood $V$ of $0\in\R^n$ it must have a
smooth inverse $\phi:=\psi\inv$. As $\psi$ is surjective, $\phi$ can
be extended globally to a (not necessarily smooth) function
$\phi:\S\to\R^n$ such that $\psi\circ\phi=\id_\S$.

The next step is to define, for every $x\in\S$, a diffeomorphic
transformation $R_x:\S\to\S$ (visually a ``rotation'') such that
$R_x(\psi(0))= x$. We can then define
\begin{align}
\label{eq:form:practical-encapsulation} 
x\mplus\delta&:=R_x(\psi(\delta)), &
y\mmnus x    &:=\phi(R_x\inv(y)).
\end{align}

Since $x\mplus0=R_x(\psi(0))=x$, Axiom~\eqref{ax:zero} holds under
this construction.  Axiom~\eqref{ax:surj} holds as we require
$\psi\circ\phi=\id_\S$ globally.  Finally, as $\phi\circ\psi=\id_V$
Axiom~\eqref{ax:inj} is fulfilled for $\delta\in V$ as required.
Axiom~\eqref{ax:triangle} depends on $\psi$ and $R_x$ and needs to be
established on a case-by-case basis.

\subsection{Lie-Groups as \mplus-Manifolds}
\label{sec:liegroups}

For connected Lie-groups \cite[Chap.\,20]{SmoothManifolds}, \ie
manifolds with a diffeomorphic group structure, the above steps are
very simple. On the one hand the mapping $\psi:\R^n\to\S$ can be
defined using the exponential map \cite[p.\,522]{SmoothManifolds},
which is a locally diffeomorphic map from a 0-neighborhood in the
Lie-algebra (a vector space diffeomorphic to $\R^n$) to a neighborhood
of the unit element in $\S$. For compact Lie-groups, the exponential
map is also surjective, with global inverse $\log$.

The transformation $R_x$ can be simply defined as $R_x(y):= x \cdot y$
(or alternatively $R_x(y):=y\cdot x$) using the group's
multiplication:
\begin{align}
\label{eq:form:liegroup-encapsulation} 
x\mplus\delta&:=x\cdot \exp(\delta), & y\mmnus x &:=\log(x\inv\cdot y).
\end{align}
Again \eqref{ax:zero}-\eqref{ax:inj} follow from the construction in Appendix \ref{sec:form:symmetric-encaps}.
Axiom \eqref{ax:triangle} reduces to whether
\begin{align}
    &\norm{(x\mplus\delta_1)\mmnus(x\mplus\delta_2)}\\
   &= \norm{\log\left( \left( x\cdot \exp\delta_2\right)^{-1} \cdot \left( x\cdot\exp\delta_1\right) \right) } \\
   &= \norm{\log\left( \exp(-\delta_2) \cdot \exp\delta_1 \right) } \\
   &\stackrel{?}{\le} \norm{\delta_1-\delta_2}.
\end{align}
We do not know of a result that would establish this fact in general,
and instead prove the inequality individually for each case.

\subsection{Cartesian Product of Manifolds}
\label{sec:form:cartesian-product}

\begin{lem}
The Cartesian product of two \mplus-manifolds $\S_1$ and $\S_2$ is a
\mplus-manifold $\S=\S_1\times\S_2$, with $V=V_1\times V_2$ and  
\begin{align}
(x_1, x_2)\mplus\matrxs{\delta_1\\ \delta_2} 
  &= (x_1\mplus_{\S_1}\delta_1, x_2\mplus_{\S_2}\delta_2) \\
(y_1, y_2)\mmnus(x_1, x_2)
  &=\matrxs{y_1\mmnus_{\S_1} x_1\\ y_2\mmnus_{\S_2} x_2}
\end{align}
for $(x_1,x_2),(y_1,y_2)\in\S:=\S_1\times\S_2$ and
$\matrxs{\delta_1\\ \delta_2}\in\R^{n_1}\times\R^{n_2}$.
\end{lem}
\begin{proof}

Smoothness of $\mplus$ and $\mmnus$ as well as Axioms \eqref{ax:zero},
\eqref{ax:surj} and \eqref{ax:inj} hold componentwise.
For Axiom~\eqref{ax:triangle}, we see that
\begin{align*}
\norm{(x\mplus\delta)\mmnus(x\mplus\eps)}^2 
&= \norm{(x_1\mplus\delta_1)\mmnus(x_1\mplus\eps_1)}^2  \nonumber\\
& \quad + \norm{(x_2\mplus\delta_2)\mmnus(x_2\mplus\eps_2)}^2\\
&\le\norm{\delta_1-\eps_1}^2+\norm{\delta_2-\eps_2}^2\\
&=\norm{\delta-\eps}^2.\qedhere
\end{align*}
\end{proof}

\subsection{Expected Value on \mplus-Manifolds}
\label{sec:form:expectedValue}

Also using Axiom~\eqref{ax:triangle}, we prove that the 
definition of the expected value by a minimization problem in \eqref{eq:bpm:E-def} 
implies the implicit definition \eqref{eq:bpm:E-implicit}:
\begin{lem}
For a random variable $X:\Omega\rightarrow\S$ and $\mu\in\S$,
\begin{equation}
\E\norm{X\mmnus\mu}^2=\min_{\mu\in\S}! \imp \E(X\mmnus\mu)=0.
\end{equation}
\end{lem}
\begin{proof}\newcommand\mmu{\mmnus\mu}\newcommand\mup{\mu\mplus}
Let $\mu:=\argmin_\mu\E\norm{X\mmnus\mu}^2$. Then
\begin{align}
\E&\norm{X\mmu}^2 
\le\E\norm{X\mmnus(\mup\E(X\mmu))}^2 \\
&=\E\norm{(\mup(X\mmu))\mmnus(\mup\E(X\mmu))}^2\\
&\le\E\norm{(X\mmu)-\E(X\mmu)}^2\\
&=\E\norm{X\mmu}^2-\norm{\E(X\mmu)}^2
\end{align}
Hence, $\E(X\mmu)=0$.
\end{proof}

\subsection{(Gaussian) Distributions on \mplus-Manifolds}
\label{sec:form:PDonM}

The basic idea of \eqref{eq:bpm:DistOnManiDef} in Section \ref{sec:bpm:PDonM} was to map a distribution $X:\Omega\to\R^n$ to a
distribution $Y:\Omega\to\S$ by defining $Y:=\mu\mplus X$ for some $\mu\in\S$.
The problem is that in general \mplus is not injective. Thus (infinitely)
many $X$ are mapped to the same $Y$, which makes even simple things such as
computing $p(Y=y)$ for a given $y\in\S$ complicated, not to mention maximizing
likelihoods.

A pragmatic approach is to ``cut off'' the distribution $X$
where $\mplus$ becomes ambiguous, \ie define a
distribution $\tX$ with
\begin{equation} \label{eq:CutOffDistributionDef}
\p(\tX = x) := \p(X = x \mid X \in V) 
\end{equation}
This can be justified because, if $V$ is large compared to the
covariance, $P(X\notin V)$ is small and the
cut-off error is negligible.  In practice, the fact that noise usually
does not really obey a normal distribution leads to a much bigger
error.

Now \eqref{eq:bpm:DistOnManiDef} simplifies to
\begin{equation}
 \p(\mu\mplus\tX=y) = \p(\tX=y\mmnus\mu),
\end{equation}
because $\mplus$ is bijective for $\tX\in V$. 
We also find that for normal distributed noise, the maximum likelihood
solution is the least squares solution.

\begin{lem}
For random variables $X:\Omega\rightarrow\S$, $Z:\Omega\rightarrow\M$, 
a measurement $z\in\M$, and a measurement function
$f:\S\rightarrow\M$, with $f(X)= z\mplus\teps$ and $\teps\sim\Ndp{0}{\Sigma}$
under the precondition $\teps\in V$,
the $x$ with largest likelihood $p(Z=z|X=x,\teps\in V)$ is the one 
that minimizes $\frac12\norm{f(x)\mmnus{}z}^2_{\Sigma}$.
\end{lem}
\begin{proof}
\begin{align}
\p(Z&=z|X=x, \teps\in V) = \\
\notag
\p(z&\mplus\teps=f(x)|\teps\in V) = \p(\teps = f(x)\mmnus{}z) \\
  &\propto \Nde{f(x)\mmnus{}z}{\Sigma}
  =\max! \\
 \eqv \;-\ln&\left(\Nde{f(x)\mmnus{}z}{\Sigma}\right) \nonumber\\
 &=\frac12\norm{f(x)\mmnus{}z}^2_{\Sigma} = \min!
\end{align}
Thus the classical approach of taking the negative log-likelihood shows the
equivalence.
\end{proof}

\end{modified}

\section{Examples of \mplus-Manifolds}\label{sec:examplemanifolds}

In this appendix we will show Axioms \eqref{eq:frm:axioms}
for the \mplus-manifolds discussed in Section~\ref{sec:bpm:orientation}
and further important examples. All, except $\R/2\pi\Z$
are based on either rotation matrices $\SO{n}$ or unit vectors $S^n$,
so we start with these general ones.
Often several representations are
possible, \ie $\R/2\pi\Z$, $SO(2)$, or $S^1$ for 2D rotations and $SO(3)$ or
$\mathbb{H}$ for 3D rotations. We will show these representations to be
isomorphic, so in particular, Axiom \eqref{ax:triangle} holds for all
if it holds for one.

\subsection{The Rotation Group \SO{n}}
\label{sec:examples:son}

Rotations are length, handedness, and origin preserving transformations
of $\R^n$. Formally they are defined as a matrix-group
\begin{equation}\notag
  \SO{n}= \setwith{Q\in\R^{n\times n}}{Q\trans Q =\I , \det Q=1}.
\end{equation}
Being subgroups of $Gl(n)$, the \SO{n} are Lie-groups. Thus we can use
the construction in \eqref{eq:form:liegroup-encapsulation}:
\begin{align}
    x\mplus{}\delta = x \exp \delta \quad y\mmnus{}x = \log\left(x^{-1}y\right)
    \label{eq:boxplus:son}
\end{align}
The matrix exponential is defined by the usual power series $\exp \delta=\sum_{i=0}^{\infty} \frac{1}{i!}\delta^i$, where the vector $\delta$
is converted to an antisymmetric matrix (we omit the $\hat{\phantom\delta}$ commonly indicating this). The logarithm is the inverse of $\exp$.
The most relevant $SO(2)$, $SO(3)$ have analytic formulas,
\cite{Cleve2003, Cardosoa2010} give general numerical algorithms.

\eqref{eq:boxplus:son} fulfills axioms \eqref{ax:zero}--\eqref{ax:inj}
for suitable $V$ by the construction using the Lie-group structure. We
conjecture that we can take $V=B_\pi(0)$, and that the remaining
axiom \eqref{ax:triangle} also holds in general.  We prove this for $n=2$ using an isomorphism to
$\R/2\pi\Z$ (App.\@~\ref{sec:examples:sotwo}) and for $n=3$ using an
isomorphism to $\mathbb{H}$ (App.\@~\ref{sec:examples:quat}).

\subsection{Directions in \texorpdfstring{$\R^{n+1}$}{R\textasciicircum(n+1)} as Unit Vectors \texorpdfstring{$S^n$}{S\textasciicircum n}}\label{sec:ChartsOfUnit-Spheres}
\label{sec:examples:sn}
Another important manifold is the unit-sphere
\begin{equation}
S^n = \setwith{x\in\R^{n+1}}{\norm x = 1},
\end{equation}
the set of directions in $\R^{n+1}$.
In general $S^n$ is no Lie-group, but we can still exploit symmetry
by \eqref{eq:form:practical-encapsulation} (Sec.\@~\ref{sec:form:symmetric-encaps}) and define a mapping $R_x$ 
that takes the first unit vector $e_1$ to $x$. This is achieved by a
Householder-reflection~\cite[Chap.\@~11.2]{Press92}.
\begin{align}
    R_x =
    \begin{cases}
        \left(\I-2\frac{vv\trans }{v\trans v}\right)X,\! &\text{for } v=x-e_1\ne0, \\
        \I, & \text{for } x=e_1\!
    \end{cases}
    \!\!\label{eq:def:rx}
\end{align}
Here $X$ is a matrix negating the second vector component. It makes
$R_x$ the product of two reflections and hence a rotation.
To define $e_1\mplus\delta$ we define $\exp$ and $\log$ for $S^n$ as
\begin{align}
    \exp\delta&=\quat{\cos\norm\delta}{\sinc\norm\delta \delta},
    \label{eq:def:expsn} \\
    \log\quat wv &= 
    \begin{cases} 
        \atantwo(0, w) e_1 & v=0\\
        \displaystyle\frac{\atantwo(\norm v, w)}{\norm v}v& v\ne0
    \end{cases}   
    \label{eq:def:logsn}
\end{align}
We call these functions $\exp$ and $\log$, because they
correspond to the usual power-series on complex numbers ($S^1$)
and quaternions ($S^3$). In general, however, there is only a
rough analogy.

Now, $S^n$ can be made a \mplus-manifold by \eqref{eq:form:practical-encapsulation},
with $\psi=\exp$ and $\phi=\psi^{-1}=\log$:
\begin{align}
    x\mplus{}\delta = R_x \exp\delta, \quad y\mmnus{}x = \log(R_x\trans y)
    \label{eq:boxplus:sn} 
\end{align}
The result looks the same as the corresponding definition
\eqref{eq:form:liegroup-encapsulation} for Lie-groups, justifying the
naming of \eqref{eq:def:expsn} and \eqref{eq:def:logsn} as $\exp$ and
$\log$. We have that $\exp$ is left inverse to $\log$, and $\log$ is
left inverse to $\exp$ on $\norm\delta<\pi$. As proved in
Lemma~\ref{lem:smooth-exp-log} (Appendix~\ref{sec:technicalProofs}),
$\exp$ and $\log$ are smooth. Hence Axioms \eqref{ax:zero},
\eqref{ax:surj}, and \eqref{ax:inj} hold for $V=B_\pi(0)$
(Sec.\@\ref{sec:form:symmetric-encaps}). Axiom~\eqref{ax:triangle} is
proved as Lemma~\ref{lem:hyperspheretriangleaxiom} in
Appendix~\ref{sec:technicalProofs}.

The induced metric $\d(x,y)$ corresponds to the angle
between $x$ and $y$ (Lemma~\ref{lem:geodetic}).

Note that the popular stereographic projection cannot be extended to
a \mplus-manifold, because it violates Axiom \eqref{ax:surj}.

\medskip

Equipped with \mplus-manifolds for $\SO{n}$ and $S^n$, we now discuss
the most important special cases, first for $n=2$, then $n=3$.

\begin{modified}
\subsection{2D Orientation as an Orthonormal Matrix}
\label{sec:examples:sotwo}

For planar rotations, $\exp$ in \eqref{eq:boxplus:son} takes an antisymmetric $2\times2$ matrix,
\ie a number and returns the well-known 2D rotation matrix
\begin{align}
    x\mplus{}\delta &= x \exp \delta, \quad y\mmnus{}x = \log\left(x^{-1}y\right) \\
\nonumber
    \exp\delta &= \matrxs{\cos\delta & -\sin\delta \\ \sin\delta & \cos\delta },
    \label{eq:boxplus:sotwo} 
    \log{}x = \atantwo(x_{21},x_{11}).     
\end{align}
The function $\exp$ is also an isomorphism between $\R/2\pi\Z$
\eqref{eq:boxplus:angle} and $SO(2)$ \eqref{eq:boxplus:sotwo}, because
the $2\pi$-periodicity of $\exp$ as a function matches the periodicity
of $\R/2\pi\Z$ as a set of equivalence classes and
$\exp{}x\cdot\exp\delta = \exp(x+\delta)$. The latter
holds as multiplication in $SO(2)$ commutes. From this argument we
see that $\R/2\pi\Z$ (Sec.\@~\ref{sec:bpm:twodrotation}) and $SO(2)$
are isomorphic \mplus-manifolds
and also that axiom \eqref{ax:triangle} holds for the latter.

\subsection{2D Orientation as a Complex Number}

Using complex multiplication, \SO2 is isomorphic to the complex
numbers of unit-length, \ie $S^2\subset\mathbb{C}$. For
$n=2$, \eqref{eq:def:expsn} and \eqref{eq:def:logsn} simplify to
\begin{align}
    \exp_{S^1}\delta 
     = \matrxs{\cos\delta\\\sin\delta},
    \log_{S^1}\matrxs{x\\y} 
      = \atantwo(y,x).
\end{align}
$R_{\matrxs{x\\y}}$ equals complex multiplication with
${\matrxs{x\\y}}$ or, as a matrix, $\matrxs{x &
  -y \\ y & x}$.  This is because $R_{\matrxs{x\\y}}$ it is a rotation
mapping $e_1$ to $\matrxs{x\\y}$ and there is only one such rotation
on $S^1$. With these prerequisites, $S^1$ is isomorphic to \SO2 by
Lemma \ref{lem:isomorphicsonesotwo} as is $\R/2\pi\Z$.

\subsection{Directions in 3D Space as \texorpdfstring{$S^2$}{S\texttwosuperior}}\label{sec:ex:S2}
The unit sphere is the most important example of a manifold that is
not a Lie-group.  This is a consequence of the ``hairy ball theorem''
\cite{Eisenberg1979}, which states that on $S^2$ every continuous
vector-field has a zero -- if $S^2$ was a Lie-group, one
could take the derivative of $x\cdot\delta$ for every $x$ and some
fixed $\delta$ to obtain a vector field on $S^2$ without zeroes.  With
the same argument applied to $x\mplus\delta$, it is also impossible to
give $S^2$ a $\mplus$-structure that is continuous in $x$. This is one
reason why we did not demand continuity in $x$ for \mplus.  It also
shows that the discontinuity in \eqref{eq:def:rx} cannot be avoided in
general (although it can for $S^1$ and $S^3$).

However, we can give a simpler analytical formula for $R_x$ in $S^2$:
\begin{gather}
\begin{split}
R_{\matrxs{x\\y\\z}} := \matrxs{
x & -r & 0\\
y & xc & -s\\
z & xs & c\\
},
\alpha=\atantwo(z,y), \\
c=\cos\alpha, s=\sin\alpha, r=\sqrt{x^2+z^2} \\
\end{split} \\
    x\mplus{}\delta = R_x  \exp\delta, \quad y\mmnus{}x = \log(R_x\trans{}y)
    \label{eq:boxplus:stwo}
\end{gather}
The formula is discontinuous for $r=0$, but any value for
$\alpha$ leads to a proper rotation matrix $R_x$. Therefore,
neither \mplus nor \mmnus are continuous in
$x$, but they are smooth with respect to $\delta$ or $y$.

\subsection{3D Orientation as a Unit Quaternion}
\label{sec:examples:quat}

The \mplus-manifold for quaternions $\mathbb{H}$ presented in Sec.\@~\ref{sec:bpm:quaternion},
\eqref{eq:quaternionboxplus:def}-\eqref{eq:log:quat} is a special case of
\eqref{eq:def:expsn}-\eqref{eq:boxplus:sn} for the unit sphere $S^3$. In
the construction $R_q$ from \eqref{eq:def:rx}
is conveniently replaced by $q^{-1}\cdot$, because $\mathbb{H}$ is a Lie-group. Also,
$\exp$ and $\log$ correspond again to the usual functions for $\mathbb{H}$.
The Axioms \eqref{ax:zero}-\eqref{ax:inj} are
fulfilled for $V=B_\pi(0)$. Axiom
\eqref{ax:triangle} is proved in Lemma \ref{lem:quat:triangle}.

The metric $\d(x,y)$ is the angle between $x$ and $y$,
and also monotonically related to the simple Euclidean metric $\norm{x-y}$ (Lemma \ref{lem:quat:euclidean}).

Orthonormal matrices $SO(3)$ and unit quaternions $S^3/\{\pm1\}$ are two different
representations of rotations. Topologically  this is called a universal
covering \cite[\textsection12]{Altmann86}. Hence, their \mplus-manifolds are isomorphic
with the usual conversion operation
\begin{align}
    \begin{split}
    \phi&\matrxs{w\\x\\y\\z} = 
        \matrxs{1-2(y^2+z^2) & -2wz + 2xy & -2wy + 2xy \\
                2wz+2xy & 1-2(x^2+z^2) & -2wx+2yz \\
                -2wy+2xz & 2wx + 2yz & 1-2(x^2+y^2)}.
    \end{split}
\end{align}
For the proof we use that the well-known expressions $\exp_{\SO3}(\alpha v/\norm{v})$ and
$\exp_{S^3}(\tfrac{1}{2}\alpha v/\norm{v})$ for the rotation 
by an angle $\alpha$ around an axis $v$ in matrix and quaternion
representation, with the exponentials defined in \eqref{eq:exp:mat}
and \eqref{eq:exp:quat}, so
\begin{align}
 \phi&(\exp_{S^3}(\delta/2)) = \exp_{\SO3}(\delta), \\
    \phi&(q\mplus_{S^3}\delta)  \\
    &= \phi(q \exp_{S^3}(\delta/2)) = \phi(q) \phi(\exp_{S^3}(\delta/2)) \\
    &= \phi(q) \exp_{\SO3}(\delta) = \phi(q) \mplus_{\SO3} \delta.
\end{align}
The isomorphism also shows Axiom \eqref{ax:triangle} for $\SO3$.

\subsection{The Projective Space as \texorpdfstring{$S^n$}{S\textasciicircum n}}
Further important manifolds, \eg in computer vision, are the projective
spaces, informally non-Euclidean spaces
where parallels intersect at infinity. Formally we define
\begin{equation}
\PP^n := \xfrac{(\R^{n+1}\setminusset0)}{(\R\setminusset0)}
\end{equation}
and write 
\begin{equation}
[x_0:x_1:\dots:x_n]=[x]= \setwith{\lambda x}{\lambda\in\R\setminusset0}
\end{equation}
for the equivalence class modulo $\R\setminusset0$ of a vector
$x=(x_1,\dots,x_n)\in\R^{n+1}\setminusset0$. In other words, $\PP^n$ is the space
of non-zero $(n+1)$-dimensional vectors modulo identification of
scalar multiples.

As every point $[x]\in\PP^n$ can uniquely be identified with the set
$\bigl\{\tfrac{x}{\norm x},\tfrac{-x}{\norm x}\bigr\}\subset S^n$ we find that
$\PP^n \cong {S^n\!}/{\{\pm1\}}$
and $S^n$ is a cover of $\PP^n$.
For $n=3$, we see the not quite intuitive fact that $\PP^3\cong\SO3$, so we can
reuse the same \mplus-manifold there.
For $\PP^2$ we can basically reuse $S^2$ but care has to
be taken due to the ambiguity of $x\equiv-x$. This can be solved by using
$\olog$ instead of $\log$ in \mmnus, fulfilling 
Axioms~\eqref{eq:frm:axioms} for $V=B_{\pi/2}$.

We cannot currently say anything about the induced metric on a
projective space.

\end{modified}


\section{Technical Proofs}
\label{sec:technicalProofs}

\begin{lem} \label{lem:smooth-exp-log}
The exponential 
$\exp\delta=\quat{\cos\norm\delta}{\sinc\norm\delta \delta}$ from
\eqref{eq:def:expsn} is analytical on $\R^n$ and $\log$ is analytical on
$S^n\setminusset{\quat{-1}{\phantom{-}0}}$.
(Therefore, also $C^\infty$)
\end{lem}
\begin{proof}
The functions $\cos$ and $\sinc$ are both globally analytic, with Taylor series
\begin{align}
\cos\norm\delta  &= \sum_{k=0}^\infty \tfrac{(-1)^k\norm\delta^{2k}}{(2k)!}
                  = \sum_{k=0}^\infty \tfrac{(-\norm\delta^2)^k}{(2k)!}, \\
\sinc\norm\delta &= \sum_{k=0}^\infty \tfrac{(-1)^k\norm\delta^{2k}}{(2k+1)!}
                  = \sum_{k=0}^\infty \tfrac{(-\norm\delta^2)^k}{(2k+1)!}.
\end{align} 
Moreover,  $\norm\delta^2=\sum_{i=1}^n\delta_i^2$ is also analytic.

On the restriction $\exp:B_\pi\to S^n\setminusset{\quat{-1}{\phantom{-}0}}$ the
inverse of $\exp$ is $\log$.
In order to prove that $\log$ is analytic, we have to show that the Jacobian
of $\exp$ has full rank.
Using that the derivative of $\norm\delta$ is $\delta\trans/\norm\delta$ and
hence the derivative of  $\cos\norm\delta$ is $\sinc\norm\delta\delta\trans$,
we show that for every $v\neq0$:
\begin{align}
\tfrac\partial{\partial\delta}(\exp\norm\delta) v 
 &=\matrxs{
    \sinc\norm\delta \delta\trans \\
    \sinc\norm\delta \I + \delta(\sinc'\norm\delta)/\norm\delta \delta\trans
 } v
\end{align}
where the first component vanishes only for $\delta\trans v=0$ (as
$\sinc x\ne0$ for $\abs x<\pi$). In this case the lower part becomes
$\sinc\norm\delta v$, which never vanishes for $\norm\delta<\pi$.
\end{proof}

\begin{lem}    For two unit quaternions $a,b\in{}S^3$, there is a monotonic mapping
    between their Euclidean distance and the distance induced by $\mmnus$:
    \begin{equation}
        \norm{a-b}^2 = f(\norm{a\mmnus{}b}), \text{with } f(\alpha)=2-2\cos(\alpha/2)
    \end{equation}
    This also holds when antipodes are identified and both metrics are
    defined as the minimum obtained for any choice of representatives
    from the equivalence classes $\{a,-a\}, \{b,-b\}$.
    \label{lem:quat:euclidean}
\end{lem}

\begin{proof}
    We put $\delta:=b\mmnus{}a=2\log(b^{-1}a)$. Then
\begin{align}
    \hskip1em&\hskip-1em\norm{b-a}^2 = \norm{\tquat10-b^{-1}a}^2 \\
    &= \norm{\tquat10-\exp(\tfrac{1}{2}\delta)}^2 \\
    &= \norm{\matrxs{1-\cos(\tfrac{1}{2}\norm{\delta})\\\sinc(\tfrac{1}{2}\norm{\delta}) \tfrac{1}{2}\delta}}^2 \\
\nonumber    &= (1-\cos(\tfrac{1}{2}\norm{\delta}))^2 + (\sinc(\tfrac{1}{2}\norm{\delta}) \; (\tfrac{1}{2}\norm{\delta}))^2 \\
    &= (1-\cos(\tfrac{1}{2}\norm{\delta}))^2 + \sin^2(\tfrac{1}{2}\norm{\delta}) \\
    &= 2 - 2\cos(\tfrac{1}{2}\norm{\delta}) 
\end{align}
   The inequality holds for every pair of antipodes $(a,b),
   (-a,b), (a,-b), (-a,-b)$ hence also for their minimum, which is by
   definition the distance of the equivalence classes.
\end{proof}

\begin{lem}
    For two orthonormal matrices $A,B\in\SO3$, there is a monotonic mapping
    between their Frobenius distance and the distance induced by $\mmnus$:
\begin{equation}
   \normF{B-A}^2 = f(\norm{A\mmnus{}B}) \text{ with }
   f(\alpha)=4-4\cos\alpha.
\end{equation}
    \label{lem:mat:frobenius}
\end{lem}
\begin{proof}
We put $\delta=A\mmnus{}B$. Then
\begin{align}
    \normF{B-A}^2 
         & =\normF{A\mplus\delta-A}^2 \\
         &= \normF{A\exp\delta-A}^2 \\
         &= \normF{\exp\delta-\I}^2.
\intertext{%
Let $Q$ be an orthonormal matrix that rotates $\delta$ into
$x$-direction, \ie $Q\delta = (\norm\delta, 0, 0)\trans$. As $\exp(Q\delta)Q
= Q\exp\delta$, we have
}%
    &= \normF{Q\trans\exp(Q\delta)Q-\I}^2 \\
    &= \normF{\exp(Q\delta){}-\I}^2 \\
    &= \normF{\exp(\norm\delta,0,0)\trans-\I}^2 \\
    &= \normF{\matrxs{0&0&0 \\ 
                     0&\cos\norm\delta-1& -\sin\norm\delta \\
                     0&\sin\norm\delta & \cos\norm\delta-1}}^2 \\ 
    &= 2(\cos\norm\delta-1)^2+2\sin^2\norm\delta \\
    &= 4-4\cos\norm\delta. \qedhere
\end{align}%
\end{proof}

\begin{lem}
    The curve $x\mplus(\lambda\delta)$, with $\lambda\in[0,1]$ and $\mplus$
    defined by \eqref{eq:boxplus:sn}, is a geodetic on $S^n$ with arc-length $\norm\delta$.
    \label{lem:geodetic}
\end{lem}

\begin{proof} We have
    \begin{align*}
        x&\mplus(\lambda\delta) = R_x \exp (\lambda\delta) = R_x \tquat{\cos(\lambda\norm\delta)}{\sinc(\lambda\norm\delta) (\lambda\delta)}\\
        &= R_x \tquat{\cos(\lambda\norm\delta)}{\sin(\lambda\norm\delta) \delta/\norm\delta} 
         = R_x \left[ \begin{smallmatrix} 1 & 0 \\ 0 & \delta/\norm\delta \end{smallmatrix} \right]
               \left[ \begin{smallmatrix} \cos(\lambda\norm\delta) \\ \sin(\lambda\norm\delta) \end{smallmatrix} \right].
    \end{align*}
    It can be seen that $x\mplus(\lambda\delta)$ is a circle segment with radius $1$ and hence a geodetic of length $\norm\delta$ on $S^n$.
\end{proof}

\begin{lem}
  For the $\mplus$ and $\mmnus$ operators on the hypersphere $S^n$ defined
  in \eqref{eq:boxplus:sn}, Axiom~\eqref{ax:triangle}
  holds~(Fig.\@~\ref{fig:triangle}).
    \label{lem:hyperspheretriangleaxiom}
\end{lem}

\begin{proof}
  By Lemma~\ref{lem:geodetic}, the expression
  $\norm{(x\mplus\delta_1)\mmnus{}(x\mplus\delta_2)}$ involves a
  triangle of three geodetics $(x, x\mplus\delta_1)$, $(x,
  x\mplus\delta_2)$, and $(x\mplus\delta_1, x\mplus\delta_2)$. By the
  same lemma, the first two have length $\alpha = \norm{\delta_1}$ and
  $\beta = \norm{\delta_2}$. Hence by the spherical law of cosines,
  with $\gamma=\angle(\delta_1,\delta_2)$, the third has a length of
\begin{align}
\hskip1em&\hskip-1em \norm{(x\mplus\delta_1)\mmnus{}(x\mplus\delta_2)} \\
    &=  \arccos\left(\cos\alpha\cos\beta + \sin\alpha\sin\beta\cos\gamma\right) \\
    \stckrel{\text{Lemma \ref{lem:triangle}}}{}{\le} \sqrt{\alpha^2+\beta^2-2\alpha\beta\cos\gamma} 
    = \norm{ \delta_1-\delta_2 }   \qedhere
\end{align}
\end{proof}

\begin{figure}
  \psfragfig{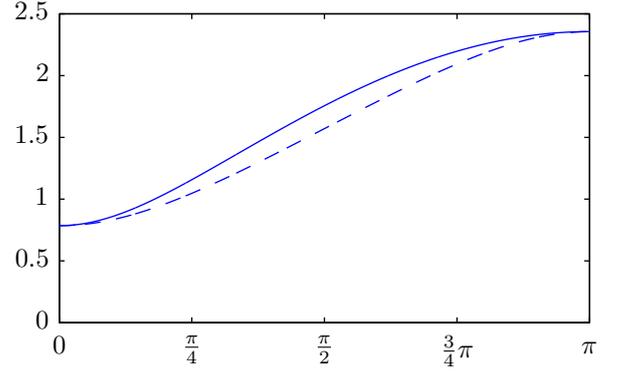}
    \vspace{-1.5\baselineskip}
    \caption{Spherical distance $\norm{(x\mplus\delta_1)\mmnus{}(x\mplus\delta_2)}$ (dashed)
            and Euclidean distance in the tangential plane $\norm{\delta_1-\delta_2}$ (solid) 
            plotted over the angle $\gamma$ between $\delta_1$ and $\delta_2$ 
            for the case $\norm{\delta_1}=\alpha=\frac{\pi}{4}$ and $\norm{\delta_2}=\beta=\frac{\pi}{2}$. 
            Confer both sides of \eqref{eq:triangle} for the concrete terms plotted.}
    \label{fig:triangle}
\end{figure}

\begin{lem}
    For the quaternion $\mplus$ and $\mmnus$ operators defined in
    \eqref{eq:quaternionboxplus:def}, Axiom
    \eqref{ax:triangle} holds~\text{(Fig.\@~\ref{fig:triangle})}.
    \label{lem:quat:triangle}
\end{lem}

\begin{proof}
   We apply the definitions and notice from \eqref{eq:log:quat} that $\norm{\olog{}q}=\arccos\abs{\Re{}q}\le\arccos(\Re{}q)$ exploiting that $\norm{q}=1$. Thus
   \begin{align}
\hskip.7em&\hskip-.7em \norm{(x\mplus\delta_1)\mmnus{}(x\mplus\delta_2)} \nonumber \\
       &= \norm{2\olog\left((x \exp\tfrac{\delta_1}{2})^{-1} \cdot (x \exp\tfrac{\delta_2}{2})\right) } \\
       &\le 2\arccos\Re\left(\exp\tfrac{-\delta_1}{2} \cdot \exp\tfrac{\delta_2}{2}\right) \\
       &= 2\arccos\Re\left( 
     \tquat{\cos(\norm{\delta_1}/2)}{-\sinc(\norm{\delta_1}/2)/2\;\delta_1} \cdot
\nonumber     \tquat{\cos(\norm{\delta_2}/2)}{\sinc(\norm{\delta_2}/2)/2\;\delta_2} 
        \right).
   \end{align}
   Now we apply the definition of quaternion multiplication
   $\tquat{r_1}{v_1}\cdot\tquat{r_2}{v_2} = \tquat{r_1r_2-v_1\trans
     v_2}{\bullet}$ and substitute $\alpha=\norm{\delta_1}/2$,
   $\beta=\norm{\delta_2}/2$, and $\gamma =
   \angle(\delta_1,\delta_2)$.  The term $v_1\trans v_2$ becomes
   $-\sin\alpha\sin\beta\cos\gamma$, and hence we can continue the
   above chain of equalities with
%
%
\begin{align}
       &= 2\arccos\left(\cos\alpha\cos\beta + \sin\alpha\sin\beta\cos\gamma\right) \\
\hskip-.2em \stckrel{\text{Lemma \ref{lem:triangle}}}{}{\le} 2\sqrt{\alpha^2+\beta^2-2\alpha\beta\cos\gamma} \\
       &= \norm{ \delta_1-\delta_2 }. \qedhere
\end{align}
\end{proof}

\begin{lem}
    The distance on a sphere is less or equal to the Euclidean distance
    in the tangential plane (Fig.\@~\ref{fig:triangle}). Formally, for all $\alpha,\beta\ge0,\gamma\in\R$
    \begin{align}
        \begin{split}
        \arccos\bigl(\cos\alpha \cos\beta + \sin\alpha\sin\beta\cos\gamma\bigr) \\ 
            \quad\le\sqrt{\alpha^2+\beta^2-2\alpha\beta\cos\gamma}.
        \end{split}
        \label{eq:triangle}
    \end{align}
    \label{lem:triangle}
\end{lem}

\begin{proof}
    If the right-hand side exceeds $\pi$, the inequality is trivial. Otherwise we
    substitute $\lambda:=\cos\gamma$ and take the cosine:
\begin{align}
    \begin{split}
     \cos\alpha \cos\beta + \sin\alpha\sin\beta\lambda \\
    \quad \ge\cos\sqrt{\alpha^2+\beta^2-2\alpha\beta\lambda}.
    \end{split}
    \label{eq:triangle:cos}
\end{align}
The proof idea is that the left-hand side of \eqref{eq:triangle:cos}
is linear in $\lambda$, the right-hand side is convex in $\lambda$, and
both are equal for $\lambda=\pm1$.  Formally, from the
cosine addition formula we get
    \begin{align}
        \label{eq:cosaddition:a}
        \cos\alpha\cos\beta - \sin\alpha\sin\beta &= \cos(\alpha+\beta) \\
         &\hspace{-1cm}= \cos \sqrt{\alpha^2+\beta^2+2\alpha\beta}, 
          \\
        \label{eq:cosaddition:b} 
        \cos\alpha\cos\beta + \sin\alpha\sin\beta &= \cos(\alpha-\beta) \\
         &\hspace{-1cm}= \cos \sqrt{\alpha^2+\beta^2-2\alpha\beta}.
    \end{align}
    Taking a $\frac{1-\lambda}{2} : \frac{1+\lambda}{2}$ convex
    combination of \eqref{eq:cosaddition:a} and
    \eqref{eq:cosaddition:b} we get the left-hand
    side of \eqref{eq:triangle:cos}:
    \begin{align}
        &\cos\alpha \cos\beta + \sin\alpha\sin\beta\lambda \\
\nonumber        
        &= \cos\alpha\cos\beta + \sin\alpha\sin\beta\left(\tfrac{1-\lambda}{2}(-1) + \tfrac{1+\lambda}{2}(+1)\right)\\        
\nonumber  \begin{split} 
               &=\phantom{+}\tfrac{1-\lambda}{2} \cos \sqrt{\alpha^2+\beta^2-2\alpha\beta} \\
               &\phantom{=}+\tfrac{1+\lambda}{2} \cos \sqrt{\alpha^2+\beta^2+2\alpha\beta}
           \end{split} \\
\nonumber        &\ge \cos \sqrt{\alpha^2+\beta^2+2\alpha\beta\left(\tfrac{1-\lambda}{2}(-1) + \tfrac{1+\lambda}{2}(+1)\right)} \\
        &= \cos \sqrt{\alpha^2+\beta^2+2\alpha\beta\lambda}.
    \end{align}
    The inequality comes from the convexity of the right-hand side of
    \eqref{eq:triangle:cos} in $\lambda$. We prove this by calculating its derivative
    \begin{gather} 
\nonumber             -\sin\sqrt{\alpha^2+\beta^2+2\alpha\beta\lambda} \;\frac{1}{2\sqrt{\alpha^2+\beta^2+2\alpha\beta\lambda}} \; 2\alpha\beta \\
        = -\sinc\sqrt{\alpha^2+\beta^2+2\alpha\beta\lambda} \; \alpha\beta
         \label{eq:triangsrhsder}
    \end{gather}
     and observing that \eqref{eq:triangsrhsder} increases monotonically until the square root exceeds $\pi$.
\end{proof}

\begin{lem}
    The following function $\phi$ is an \mplus-isomorphism between $S^1$ and $SO(2)$:
   \begin{align}
       \phi: S^1 \rightarrow \SO2, \quad 
       \phi\matrxs{x\\y} = \matrxs{x & -y \\ y & x}.
   \end{align}
   \label{lem:isomorphicsonesotwo}
\end{lem}

\begin{proof}
The map is bijective, because all matrices in $\SO2$ are of the form \eqref{eq:boxplus:sotwo}.
It also commutes with \mplus, since
\begin{align}
    \phi(\matrxs{x\\y}\mplus_{S^1}\delta) 
     &= \phi(R_{\matrxs{x\\y}}\trans\exp\delta) \\
     &= \phi\left(\matrxs{x&-y\\y&x} \matrxs{\cos\delta\\\sin\delta} \right) \\
     &= \phi\left(\matrxs{x\cos\delta-y\sin\delta\\y\cos\delta+x\sin\delta}\right) \\
     &= \matrxs{x\cos\delta-y\sin\delta & -y\cos\delta-x\sin\delta \\ y\cos\delta+x\sin\delta & x\cos\delta-y\sin\delta} \\
     &= \matrxs{x&-y\\y&x} \matrxs{\cos\delta & -\sin\delta \\ \sin\delta & \cos\delta} \\
     &= \matrxs{x&-y\\y&x} \exp{\delta} \\
     &= \phi\left(\matrxs{x\\y}\right)\mplus_{SO^2}\delta. \qedhere
\end{align}
\end{proof}

\begin{modified}
\section{Comparison to SPMap}
\label{sec:SPmap}

In an SPMap~\cite{Castellanos99} (originally 2D but extended to 3D)
every geometric entity is represented by a reference pose (\SE3) in an
arbitrary, potentially overparametrized representation, and a
perturbation vector parametrizing the entity relative to its reference
pose in minimal parametrization. The estimation algorithm operates
solely on the perturbation vector. This corresponds to
$s\mplus\delta$, with $s$ being the reference pose and $\delta$ the
perturbation vector. Actually, this concept and the idea that in most
algorithms $\mplus$ can simply replace $+$ motivated the
axiomatization of $\mplus$-systems we propose. Our contribution is to
give this idea, which has been around for a while, a thorough
mathematical framework more general than geometric entities.

If a geometric entity is ``less than a pose'', e.g.\@ a point, SPMap still uses a
reference pose but the redundant DOFs, e.g.\@ rotation, are
removed from the perturbation vector. In our 
axiomatization the pose would simply be an overparametrization of a point.
Using the notation of \cite{Castellanos99}
\begin{align}
    s \mplus \delta = s \oplus (B\trans \delta),
\end{align}
where $B$ is the so-called binding matrix that maps entries of $\delta$ to
the DOF of a pose, and $\oplus$ is the concatenation of poses.

Analogously, the SPMap represents, e.g. a line as the pose's $x$-axis
with $x$-rotation and translation being redundant DOFs removed from the
perturbation vector. Here lies a theoretical difference. Consider two
such poses differing by an $x$-rotation. They represent the same line
but differ in the effect of the perturbation vector, as $y$- and $z$-axes
point into different directions. For us, the $\mplus$-system $\S$ would
be a space of equivalence classes of poses. However, $\mplus$ maps
from $\S\times\R^n$ to $\S$, so $s\mplus\delta$ must formally be the
equivalent for equivalent poses $s$. The SPMap representation has the
advantage that it is continuous both in the pose and the perturbation
vector, which is not possible with our axiomatization due to the hairy
ball theorem as discussed in Sec.~\ref{sec:ex:S2}.

Overall, our contribution is the axiomatized and more general view,
not limited to quotients of $\SE3$ as with the SPMap. We currently investigate
axiomatization of SPMap's idea to allow different representatives of
the same equivalence class to define different $\mplus$-co-ordinate systems.
We avoided this here, because it adds another level of conceptual complexity.
\end{modified}


\bibliographystyle{model1b-num-names}
\bibliography{literature}


\end{document}